\documentclass[12pt]{article}

\usepackage{amsmath}
\usepackage{amsthm}
\usepackage{amssymb}
\usepackage{bbm}
\usepackage{verbatim}
\usepackage{graphicx}
\usepackage{subfigure}
\usepackage{afterpage}
\usepackage{etoolbox}
\usepackage{float}
\usepackage{rotating}
\usepackage[inline]{enumitem}

\usepackage{algorithmic}
\usepackage{algorithm}

\makeatletter
\renewcommand{\p@subfigure}{\thefigure}
\makeatother

\newtheorem{theorem}{Theorem}[section]
\newtheorem{proposition}[theorem]{Proposition}

\newcommand{\abs}[1]{\left | #1 \right |}

\newcommand {\defeq}{\triangleq}

\newcommand {\myvec}[1] {{\mbox{\boldmath $#1$}}}

\usepackage[english]{babel}
\usepackage[T1]{fontenc}

\usepackage[a4paper,top=3cm,bottom=2cm,left=3cm,right=3cm,marginparwidth=1.75cm]{geometry}
\usepackage{longtable}
\usepackage{natbib}
\usepackage{amsmath}
\usepackage{graphicx}
\usepackage[colorinlistoftodos]{todonotes}
\usepackage[colorlinks=true, allcolors=blue]{hyperref}

\title{Geometry-Based Data Generation}
\author{Ofir Lindenbaum$^{1,2 \dagger}$ \and Jay S. Stanley III$^{1,3 \dagger}$ \and Guy Wolf$^{2\ddagger}$ \and Smita Krishnaswamy$^{1,4\ddagger \ast}$\\
\\
\normalsize{$^{1}$Department of Genetics;} 
\normalsize{$^{2}$Applied Mathematics Program;}\\
\normalsize{$^{3}$Computational Biology \& Bioinformatics Program;}\\
\normalsize{$^{4}$Department of Computer Science;}\\
\normalsize{Yale University, New Haven, CT, USA}\\
\\
\normalsize{$^\ast$Corresponding author. E-mail: smita.krishnaswamy@yale.edu }\\\normalsize{Address: 333 Cedar St, New Haven, CT 06510, USA}\\
\normalsize{$^\dagger$ These authors contributed equally.}
\normalsize{$^\ddagger$ These authors contributed equally.}
}

\begin{document}
\maketitle
\begin{abstract}
We propose a new type of generative model of high-dimensional data that learns a manifold geometry of the data, rather than density, and can generate points evenly along this manifold. This is in contrast to existing generative models that represent data density, and are strongly affected by noise and other artifacts of data collection. We demonstrate how this approach corrects sampling biases and artifacts, and improves several downstream data analysis tasks such as clustering and classification. Finally, we demonstrate that this approach is especially useful in biology where, despite the advent of single-cell technologies, rare subpopulations and gene-interaction relationships are affected by biased sampling. We show that SUGAR can generate hypothetical populations and reveal intrinsic patterns and mutual-information relationships between genes on a single cell dataset of hematopoiesis. 
\end{abstract}

\section{Introduction}
Manifold learning methods in general, and diffusion geometry ones in particular~\citep{lafon:DM}, are traditionally used to infer latent representations that capture intrinsic geometry in data, but they do not relate them to original data features. Here, we propose a novel data synthesis method, which we call SUGAR (Synthesis Using Geometrically Aligned Random-walks), for generating data in its original feature space while following its intrinsic geometry. This geometry is inferred by a diffusion kernel that captures a data-driven manifold and reveals underlying structure in the full range of the data space -- including undersampled regions that can be augmented by new synthesized data. Geometry-based data generation with SUGAR is motivated by numerous uses in data exploration. For instance, in biology, despite the advent of single-cell technologies such as single-cell RNA sequencing and mass cytometry, sampling biases and artifacts often make it difficult to evenly sample the dataspace.  Rare populations of relevance to disease and development are often left out~\citep{grun2015single}. By learning the data geometry rather than density, SUGAR is able to generate hypothetical cell types for exploration. 

Further, imbalanced data is problematic for many machine learning applications. Class density can strongly bias some classifiers~\citep{he2009learning, lopez2013insight, hensman2015impact}. In clustering, imbalanced sampling of ground truth clusters can lead to distortions in learned clusters~\citep{xuan2013exploring,wu2012uniform}. Sampling biases can also corrupt regression tasks; relationship measures such as mutual information are heavily weighted by density estimates and thus may mis-quantify the strength of dependencies with data whose density is concentrated in a particular region of the relationship~\citep{krishnaswamy2014conditional}. SUGAR can aid such machine learning algorithms by generating data that is balanced along its manifold. 


There are several advantages of our approach over contemporary generative models. Most other generative models attempt to learn and replicate the density of the data; this approach is intractable in high dimensions. Distribution-based generative models typically require vast simplifications such as parametric forms or restriction to marginals in order to become tractable. Examples for such methods include Gaussian Mixture Models (GMM)~\citep{rasmussen2000infinite}, variational Bayesian methods~\citep{bernardo2003variational}, and kernel density estimates~\citep{scott2008kernel}. In contrast to these methods, SUGAR does not rely on high dimensional probability distributions or parametric forms. SUGAR selectively generates points to equalize density; as such, the method can be used generally to compensate for sparsity and heavily biased sampling in data in a way that is agnostic to downstream application. In other words, whereas more specialized methods may use prior information (e.g., labels) to correct class imbalances for classifier training~\citep{smote}, SUGAR does not require such information and can apply in cases (such as for clustering or regression) even when such information does not exist.

Here, we construct SUGAR from diffusion methods and theoretically justify the density equalization properties of SUGAR. We then demonstrate SUGAR on imbalanced artificial data. Subsequently, we use SUGAR to improve classification accuracy on 61 imbalanced datasets. We then provide an illustrative synthetic example of clustering with SUGAR and show the clustering performance of the method on 115 imbalanced datasets obtained from the KEEL-dataset repository~\citep{alcala2009keel}.  Finally, we use SUGAR for exploratory analysis of a biological dataset, recovering imbalanced cell types and restoring canonical gene-gene relationships.

\section{Related Work}
Most existing methods for data generation assume a probabilistic data model. Parametric density estimation methods such as~\citep{rasmussen2000infinite,varanasi1989parametric} find a best fitting parametric model for the data using maximum likelihood, which is then used to generate new data. Nonparametric density estimators~\citep{seaman1996evaluation,scott1985averaged,gine2002rates} use a histogram or a kernel~\citep{scott2008kernel} to estimate the generating distribution. Recently, Variational Auto-Encoders~\citep{kingma2013auto,doersch2016tutorial} and Generative Adversarial Nets~\citep{goodfellow2014generative} have been demonstrated for generating new points from complex high dimensional distributions. 

A family of manifold based Parzen window estimators are presented in~\citep{parzen1,parzen2,parzen3}. These methods exploit the manifold's structure to improve a density estimation of high dimensional data. Markov-chain Monte Carlo on implicitly defined manifolds was presented in~\citep{mmcmc1,mmcmc2}. The authors use implicit constrains to generate new points that follow a manifold structure. Another scheme by \citet{oztireli2010spectral} defines a spectral measure to resample existing points such that manifold structure is preserved and the density of points is uniform. These methods differ from the proposed approach as they either require implicit constraints or they change the values of existing points in the resampling process.

\section{Background}
\label{sec:Background} 
\subsection{Diffusion Geometry}
\label{SecDiff}
\label{sec:construction}
\citet{lafon:DM} proposed the non-linear dimensionality reduction framework called Diffusion Maps (DM).  This popular method robustly captures an intrinsic manifold geometry using a row-stochastic Markov matrix associated with a graph of the data.  This graph is commonly constructed using a Gaussian kernel
\begin{equation}
\label{GKernel}
{\cal{K}}(\myvec{x}_i,\myvec{x}_j)\defeq K_{i,j}=exp\left( {-\frac{\|\myvec{x}_i-\myvec{x}_j\|^2}{2
        \sigma^2}  }\right),i,j=1,...,N
\end{equation}
where $x_1,\ldots, x_N$ are data points, and $\sigma$ is a bandwidth parameter that controls neighborhood sizes. Then, a diffusion operator is defined as the row-stochastic matrix ${P_{i,j}={\cal{P}}(\myvec{x}_i,\myvec{x}_j)=[{{\myvec{D}}^{-1}{\myvec{K}}}}]_{i,j}$, $i,j=1,...,N$, where $\myvec{D}$ is a diagonal matrix with values corresponding to the degree of the kernel $D_{i,i}=\hat{d}(i)=\sum_j {\cal{K}}(\myvec{x}_i,\myvec{x}_j)$. The degree $\hat{d}(i)$ of each point $\myvec{x}_i$ encodes the total connectivity the point has to its neighbors. The Markov matrix $\myvec{P}$ defines an imposed diffusion process, shown by \citet{lafon:DM} to efficiently capture the diffusion geometry of a manifold $\cal{M}$. 

The DM framework may be used for dimensionality reduction to embed data using the eigendecomposition of the diffusion operator. However, in this paper, we do not directly use the DM embedding, but rather a variant of the operator $\myvec{P}$ that captures diffusion geometry. In Sec. \ref{sec:sugar}, we explain how this operator allows us to ensure the data we generate follows diffusion geometry and the manifold structure it represents.

\subsection{Measure Based Gaussian Correlation}
\label{sec:MGC}
\citet{wolf:MGC} suggest the Measure-based Gaussian Correlation (MGC) kernel as an alternative to the Gaussian kernel (Eq.~\ref{GKernel}) for constructing diffusion geometry based on a measure $\mu$. The measure could be provided in advance or approximated based on the data samples. The MGC kernel with a measure $\mu (\myvec{r})$, $r \in \myvec{X}$, defined over a set $\myvec{X}$ of reference points, is $\hat{\cal{K}}(\myvec{x}_i,\myvec{x}_j)=\sum_{\myvec{r} \in \myvec{X}} {\cal{K}}(\myvec{x}_i,\myvec{r}){\cal{K}}(\myvec{r},\myvec{x}_j) \mu (\myvec{r})$, $i,j=1,...,N$, where the kernel $\cal{K}$ is some decaying symmetric function. Here, we use a Gaussian kernel for $\cal{K}$ and a sparsity-based measure for $\mu$. 



\subsection{Kernel Bandwidth Selection}
The choice of kernel bandwidth $\sigma$ in Eq.~\ref{GKernel} is crucial for the performance of Gaussian-kernel methods. For small values of $\sigma$ the resulting kernel $\cal{K}$ converges to the identity matrix; inversely, large values of $\sigma$ yield the all-ones matrix. Many methods have been proposed for tuning $\sigma$. A range of values is suggested in~\citep{singer2009detecting} based on an analysis of the sum of values in $\cal{K}$. \citet{lindenbaum2017kernel} presented a kernel scaling method that is well suited for classification and manifold learning tasks. We describe here two methods for setting the bandwidth: a global scale suggested in~\citep{Keller} and an adaptive local scale based on~\citep{GKernelZelnik}.

For degree estimation we use the max-min bandwidth~\citep{Keller} as it is simple and effective. The max-min bandwidth is defined by $\sigma^2_{\text{MaxMin}}={\cal{C}}\cdot \underset{j}{\max} [ \underset{i,i\neq j}{\min} (\|\myvec{x}_i-\myvec{x}_j\|^2)]$, where ${\cal{C}} \in [2,3]$. This approach attempts to force that each point is connected to at least one other point. This method is simple, but highly sensitive to outliers. \citet{GKernelZelnik} propose adaptive bandwidth selection. At each point $\myvec{x}_i$, the scale $\sigma_i$ is chosen as the $L^1$ distance of $\myvec{x}_i$ from its $r$-th nearest neighbor. This adaptive bandwidth guarantees that at least half of the points are connected to $r$ neighbors. Since an adaptive bandwidth obscures density biases, it is more suitable for applying the resulting diffusion process to the data than for degree estimation.


\section{Data Generation}
\label{sec:sugar}

\subsection{Problem Formulation}
Let $\mathcal{M}$ be a $d$ dimensional manifold that lies in a higher dimensional space $\mathbbm{R}^D$, with $d < D$, and let $\myvec{X} \subseteq \mathcal{M}$ be a dataset of $N = \abs{\myvec{X}}$ data points, denoted $\myvec{x}_1,\ldots,\myvec{x}_N$, sampled from the manifold. In this paper, we propose an approach that uses the samples in $\myvec{X}$ in order to capture the manifold geometry and generate new data points from the manifold. In particular, we focus on the case where the points in $\myvec{X}$ are unevenly sampled from $\mathcal{M}$, and aim to generate a set of $M$ new data points $\myvec{Y}=\{ \myvec{y}_1,...,\myvec{y}_M \} \subseteq \mathbbm{R}^{D}$ such that
\begin{enumerate*}
	\item the new points $\myvec{Y}$ approximately lie on the manifold $\mathcal{M}$, and
	\item the distribution of points in the combined dataset $\myvec{Z} \defeq \myvec{X} \cup \myvec{Y}$ is uniform.
\end{enumerate*}
Our proposed approach is based on using an intrinsic diffusion process to robustly capture a manifold geometry from $\myvec{X}$ (see Sec.~\ref{sec:Background}). Then, we use this diffusion process to generate new data points along the manifold geometry while adjusting their intrinsic distribution, as explained in the following sections.

\subsection{SUGAR: Synthesis Using Geometrically Aligned Random-walks}
\label{sec:Method}

SUGAR initializes by forming a Gaussian kernel $\myvec{G_X}$ (see Eq.~\ref{GKernel}) over the input data $\myvec{X}$ in order to estimate the degree $\hat{d}(i)$ of each $\myvec{x}_i \in \myvec{X}$.  Because the space in which degree is estimated impacts the output of SUGAR, $\myvec{X}$ may consist of the full data dimensions or learned dimensions from manifold learning algorithms.  We then define the sparsity of each point $\hat{s}(i)$  via
\begin{equation}
\label{eq:sparsity1}
\hat{s}(i)\defeq [{\hat{d}(i)}]^{-1},i=1,...,N.
\end{equation}
Subsequently, we sample $\hat{\ell}(i)$ points $\myvec{h}_j \in \myvec{H}_i, j = 1,..., \hat{\ell}(i)$ around each $\myvec{x}_i \in \myvec{X}$ from a set of localized Gaussian distributions $G_i = {\cal{N}}(\myvec{x}_i,\myvec{\Sigma}_i) \in \mathcal{G}$. The choice of $\hat{\ell}(i)$ based on the density (or sparsity) around $\myvec{x}_i$ is discussed in Sec.~\ref{sec:dense}. This construction elaborates local manifold structure in meaningful directions by \begin{enumerate*} \item compensating for data sparsity according to $f(\hat{s}(i))$, and \item centering each $G_i$ on an existing point $\myvec{x}_i$ with local covariance $\myvec{\Sigma}_i$ based on the $k$ nearest neighbors of $\myvec{x}_i$. \end{enumerate*} The set of all $M = \sum_i\hat{\ell}(i)$ new points, $\myvec{Y}_0 = \{\myvec{y}_1,...,\myvec{y}_M\}$, is then given by the union of all local point sets $\myvec{Y}_0 =  \myvec{H}_1 \cup \myvec{H}_2 \cup ... \cup \myvec{H}_N$. 
Next, we construct a sparsity-based MGC kernel (see Sec.~\ref{sec:MGC}) $\hat{\cal{K}}(\myvec{y}_i,\myvec{y}_j)=\sum_{r} {\cal{K}}(\myvec{y}_i,\myvec{x}_{r}){\cal{K}}(\myvec{x}_{r},\myvec{y}_j) \hat{s} (r)$ using the affinities in the sampled set $\myvec{X}$ and the generated set $\myvec{Y}_0$.  We use this kernel to pull the new points $\myvec{Y}_0$ toward the sparse regions of the manifold $\cal{M}$ using the row-stochastic diffusion operator $\myvec{\hat{{P}}}$ (see Sec.~\ref{sec:construction}). 
We then apply the powered operator $\myvec{\hat{{P}}}^t$ to $\myvec{Y}_0$, which averages points in $\myvec{Y}_0$ according to their neighbors in $\myvec{X}$.

The powered operator $\myvec{\hat{{P}}}^t$ controls the \textit{diffusion distance} over which points are averaged; higher values of $t$ lead to wider local averages over the manifold. The operator may be modeled as a low pass filter in which higher powers decrease the cutoff frequency. Because $\myvec{Y}_0$ is inherently noisy in the ambient space of the data, $ \myvec{Y}_t = \myvec{\hat{{P}}}^t \cdot \myvec{Y}_0$ is a denoised version of $\myvec{Y}_0$ along $\cal{M}$. The number of steps required can be set manually or using the Von Neumann Entropy as suggested by~\citet{moon2017visualizing}. Because the filter $\myvec{\hat{{P}}}^t \cdot \myvec{Y}_0$ is not power preserving, $\myvec{Y}_t$ is rescaled to fit the range of original values of $\myvec{X}$. A full description of the approach is given in Alg.~\ref{alg:DataGeneration}.

\begin{algorithm}[!tb] 
	\caption{SUGAR: Synthesis Using Geometrically Aligned Random-walks} \textbf{Input:} Dataset
	$\myvec{X} = \{\myvec{x}_1, \myvec{x}_2, \ldots, \myvec{x}_N\}, \myvec{x}_i \in \mathbb{R}^{D}$.\\
    \textbf{Output:}~Generated set of points $\myvec{Y} = \{\myvec{y}_1, \myvec{y}_2, \ldots, $
    $\myvec{y}_M\}, \myvec{y}_i \in \mathbb{R}^{D}$.~%
	\begin{algorithmic}[1]    
      \STATE Compute the diffusion geometry operators $\myvec{K}$, $\myvec{P}$, and degrees $\hat{d}(i)$, $i=1,...,N$
      (see Sec.~\ref{sec:Background}) 
        \STATE Define a sparsity measure $\hat{s}(i),i=1,...,N$ (Eq.~\ref{eq:sparsity1}).
		\STATE Estimate a local covariance $\myvec{\Sigma}_i$, $i=1,...,N$, using $k$ nearest neighbors around each $\myvec{x}_i$. 
        \STATE For each point $i=1,...,N$ draw $ \hat{\ell}(i) $ vectors (see Sec.~\ref{sec:dense}) from a Gaussian distribution ${\cal{N}}(\myvec{x}_i,\myvec{\Sigma}_i)$. Let $\hat{\myvec{Y}}_0$ be a matrix with these $M = \sum^N_{i=1}\hat{\ell}(i)$ generated vectors as its rows.
        \STATE Compute the sparsity based diffusion operator $\hat{\myvec{P}}$ (see Sec~\ref{sec:Method}).
        \STATE Apply the operator $\hat{\myvec{P}}$ at time instant $t$ to the new generated points in $\hat{\myvec{Y}}_0$ to get diffused points as rows of ${\myvec{Y}}_t =\hat{\myvec{P}}^t \cdot {\myvec{Y}}_0$.
    \STATE Rescale $\myvec{Y}_t$ to get the output $\myvec{Y}[\cdot,j] = \myvec{Y}_t[\cdot,j] \cdot \frac{\text{percentile}(\myvec{X}[\cdot,j], .99)}{\max{\myvec{Y}_t[\cdot,j]}}$, $j=1,\ldots,D$, in order to fit the original range of feature values in the data.
	\end{algorithmic}
	\label{alg:DataGeneration}
\end{algorithm}


\subsection{Manifold Density Equalization}
\label{sec:dens}
\label{sec:dense}

The generation level $\hat{\ell}(i)$ in Alg.~\ref{alg:DataGeneration} (step 4), i.e., the amount of points generated around each $\myvec{x}_i$, determines the distribution of points in $\myvec{Y}_1$. Given a biased dataset $\myvec{X}$, we wish to generate points in sparse regions such that the resulting density over $\cal{M}$ becomes uniform. To do this we have proposed to draw $\hat{\ell}(i)$ points around each point $\myvec{x}_i,i=1,...,N$ from ${\cal{N}}(\myvec{x}_i,\myvec{\Sigma}_i)$ (as described in Alg.~\ref{alg:DataGeneration}). The following proposition provides bounds on the ``correct'' number of points $\hat{\ell}(i),i=1,...,N$ required to balance the density over the manifold by equalizing the degrees $\hat{d}(i)$.
\begin{proposition} 
\label{prop}
The generation level $\hat{\ell}(i)$ required to equalize the degree $\hat{d}(i)$, is bounded by
$$
\det\left(\myvec{I} + \frac{ \myvec{\Sigma}_i}{2\sigma^2}\right)^{\frac{1}{2}} \frac{\max(\hat{d}(\cdot))-\hat{d}(i)}{\hat{d}(i)+1} - 1 \leq \; \hat{\ell}(i) \; \leq \det\left(\myvec{I} + \frac{ \myvec{\Sigma}_i}{2\sigma^2}\right)^{\frac{1}{2}} [{\max(\hat{d}(\cdot))-\hat{d}(i)}]\,,
$$
where $\hat{d}(i)$ is the degree value at point $\myvec{x}_i$, $\sigma^2$ is the bandwidth of the kernel $\myvec{K}$ (Eq.~\ref{GKernel}) and $\myvec{\Sigma}_i$ is the covariance of the Gaussian designed for generating new points (as described in Algorithm \ref{alg:DataGeneration}).   

\end{proposition}

\begin{proof}

	Given $\myvec{x}_i \in \mathbb{R}^D,i=1,...,N$, the degree at point $\myvec{x}_i$ (section 4.2). SUGAR (Algorithm 1) generates $\hat{\ell}(i)$ points around $\myvec{x}_i$ from a Gaussian distribution $N(\myvec{x}_i,\myvec{\Sigma})$. After generating $\hat{\ell}(i)$ points, the degree at point $\myvec{x}_i$ is defined as
	\begin{equation} \label{eq:Nkde_Full}
	\bar{d}(i) =\sum^N_{j=1} e^ {-\frac{||\myvec{x}_i-\myvec{x}_j||^2}{2\sigma^2}  }  + \sum^{\hat{\ell}(i)}_{\ell=1} e^{-\frac{||\myvec{x}_i-\myvec{y}^i_{\ell}||^2}{2\sigma^2}  } + \sum_{\tilde{\ell} } e^{-\frac{||\myvec{x}_i-\myvec{y}_{\tilde{\ell}}||^2}{2\sigma^2}  } .\end{equation}
	In order to equalize the distribution, we require that the expectation of the degree be independent of $\myvec{x}_i$, thus, we set $\mathrm{E} \Big[ \bar{d}(i) \Big] =\text{C},i=1,...,N$, where $\text{C}$ is a constant that will be addressed later in the proof. Further, we notice that the first term on the right hand side of Eq.~$\ref{eq:Nkde_Full}$ can be substituted by $\hat{d}(i)$, while the second term is a sum of $\hat{\ell}(i)$ random variables. The third term accounts for the influence of points generated around $\myvec{x}_j,j \neq i$ on the degree at point $\myvec{x}_i$. It is hard to find a closed form solution for the number of points required at each $i$, however, based on two simple assumptions we can derive bounds on the required value. First, to find an upper bound, we assume that the points generated around $\myvec{x}_i$ have negligible effect on the $\hat{d}(j)$ at point $j \neq i$. This leads to an upper bound, because in practice the degree is affected by other points as well. We note that this assumption should hold with proper choices of $\sigma$ and $\myvec{\Sigma}_i,i=1,...,N$. By substituting the constant $C$, the term $\hat{d}(i)$, and using the independence assumption, the expectation of Eq.~\ref{eq:Nkde_Const} can be written as
	\begin{equation} \label{eq:Nkde_Const}
	\text{C}-\hat{d}(i) =\mathrm{E} \Big[  \sum^{\hat{\ell}(i)}_{\ell=1} e^{-\frac{||\myvec{x}_i-\myvec{y}^i_{\ell}||^2}{2\sigma^2}  } \Big].\end{equation}
	Then, since the variables $\myvec{y}^i_{\ell}$ are i.i.d., then the right hand side of Eq.~\ref{eq:Nkde_Const} becomes
	\[ \hat{\ell}(i)  \mathrm{E} \Big[ e^{-\frac{||\myvec{x}_i-\myvec{y}^i||^2}{2\sigma^2}  }\Big]= \] 
	\[\frac{\hat{\ell}(i)}{\det(2\pi \myvec{\Sigma}_i)^{1/2}} \int_{\mathbb{R}^D} { e^{-\frac{||\myvec{x}_i-\myvec{y}^i||^2}{2\sigma^2}  }} e^{- {(\myvec{x}_i-\myvec{y}^i)^T\myvec{\Sigma}^{-1}_i (\myvec{x}_i-\myvec{y}^i ) }} d \myvec{y}^i, \]
	and then using a change of variables $\tilde{\myvec{y}}=\myvec{y}^i-\myvec{x}_i$, the integral simplifies to
	\[ \frac{\hat{\ell}(i)}{\det(2\pi \myvec{\Sigma}_i)^{1/2}} \int_{\mathbb{R}^D} e^{- {(\myvec{x}_i-\myvec{y}^i)^T(\myvec{\Sigma}^{-1}_i + \frac{\myvec{I}}{2\sigma^2})(\myvec{x}_i-\myvec{y}^i ) }} d \myvec{y}^i= \]
	\[ \frac{\hat{\ell}(i)} { \det(\myvec{I} + \frac{\det(\myvec{\Sigma}_i )}{2\sigma^2})^{0.5}  }, \]
	where the final step is based on the integral of a Gaussian and a change of variables. Finally, this leads to an upper bound on the number of points $\hat{\ell}(i)$ is $\hat{\ell}(i) \leq \det(\myvec{I} + \frac{\det(\myvec{\Sigma}_i )}{2\sigma^2})^{0.5} [C- \hat{d}(i)]  $. By choosing $C=\max(\hat{d}(\cdot))$ we guarantee that $\hat{\ell}(i)\geq 0$ (which means that we are not removing points). 
	
	Now, by considering the third term on the right side of Eq. \ref{eq:Nkde_Full} we derive a lower bound on $\hat{\ell}(i)$. The term $\sum_{\tilde{\ell} } e^{-\frac{||\myvec{x}_i-\myvec{y}_{\tilde{\ell}}||^2}{2\sigma^2}}$ accounts for points generated around neighbor points around $\myvec{x}_i$. The number of points in the region of $||\myvec{x}_i-\myvec{x}_j||^2 \sim \sigma^2$ is proportional to the degree $\hat{d}(i)$. For the lower bound, we assume that the number of points generated by each neighbor of $\myvec{x}_i$ is no more than $\hat{\ell}(i)+1$. This assumption becomes realistic by imposing smoothness on the function $\hat{\ell}(i)$. Plugging this assumption and the degree value as a measure of connectivity into Eq. \ref{eq:Nkde_Full} we get that
	\[ \bar{d}(i) \leq \hat{d}(i)  + \sum^{\hat{\ell}(i)}_{\ell=1} e^{-\frac{||\myvec{x}_i-\myvec{y}^i_{\ell}||^2}{2\sigma^2}  } + \hat{d}(i)\sum^{\hat{\ell}(i)+1}_{\tilde{\ell}=1 } e^{-\frac{||\myvec{x}_i-\myvec{y}_{\tilde{\ell}}||^2}{2\sigma^2}  },  \] we now take the expectation of both sides
	\[C \leq \hat{d}(i)  + \mathrm{E} \Big[  \sum^{\hat{\ell}(i)}_{\ell=1} e^{-\frac{||\myvec{x}_i-\myvec{y}^i_{\ell}||^2}{2\sigma^2}  } \Big] + \hat{d}(i) \mathrm{E} \Big[ \sum^{\hat{\ell}(i)+1}_{\tilde{\ell}=1 } e^{-\frac{||\myvec{x}_i-\myvec{y}_{\tilde{\ell}}||^2}{2\sigma^2}  } \Big] \leq \hat{d}(i)  + [\hat{d}(i)+1] \mathrm{E} \Big[ \sum^{\hat{\ell}(i)+1}_{{\ell}=1 } e^{-\frac{||\myvec{x}_i-\myvec{y}^i_{{\ell}}||^2}{2\sigma^2}  } \Big],  \] by replacing $\myvec{y}_{\tilde{\ell}}$ by $\myvec{y}^i_{\tilde{\ell}}$ we increase the expectation value. Finally, exploiting the i.i.d assumption and the Gaussian distribution we conclude that by replacing the mean  with similar steps as used for the upper bound we conclude that 
	\[ \det(\myvec{I} + \frac{ \myvec{\Sigma}_i}{2\sigma^2})^{0.5}   [{\max(\hat{d}(\cdot))-\hat{d}(i)}]/[\hat{d}(i)+1] -1 \leq \hat{\ell}(i). \]
	\qedhere
	

\end{proof}
In practice we suggest to use the mean of the upper and lower bound to set the number of generated points $\hat{\ell}(i)$. In Sec.~\ref{sec:densexp} we demonstrate how the proposed scheme enables density equalization using few iterations of SUGAR. 

\section{Experimental Results}

\subsection{Density Equalization}
\label{sec:densexp}
\color{black} Given the circular manifold recovered in Sec. \ref{sec:mnist}, we next sought to evaluate the density equalization properties proposed in Sec. \ref{sec:dense}.
We begin by sampling one hundred points from a circle such that the highest density is at the origin ($\theta=0$) and the density decreases away from it (Fig.~\ref{fig:circ1}, colored by degree $\hat{d}(i)$). SUGAR was then used to generate new points based on $\hat{\ell}(i)$ around each original point (Fig.~\ref{fig:circZt0}, before diffusion,~\ref{fig:circZt1}, after diffusion). We repeat this process for different initial densities and evaluate the resulting distribution of point against the amount of iteration of SUGAR. We perform a Kolmogorov-Smirnov (K-S) test to determine if the points came from a uniform distribution. The resulting P-values are presented in Fig.~\ref{fig:p_val}.
\begin{figure}[!htb]
\centering
\subfigure[]{\label{fig:circ1} \includegraphics[width=0.22\textwidth] {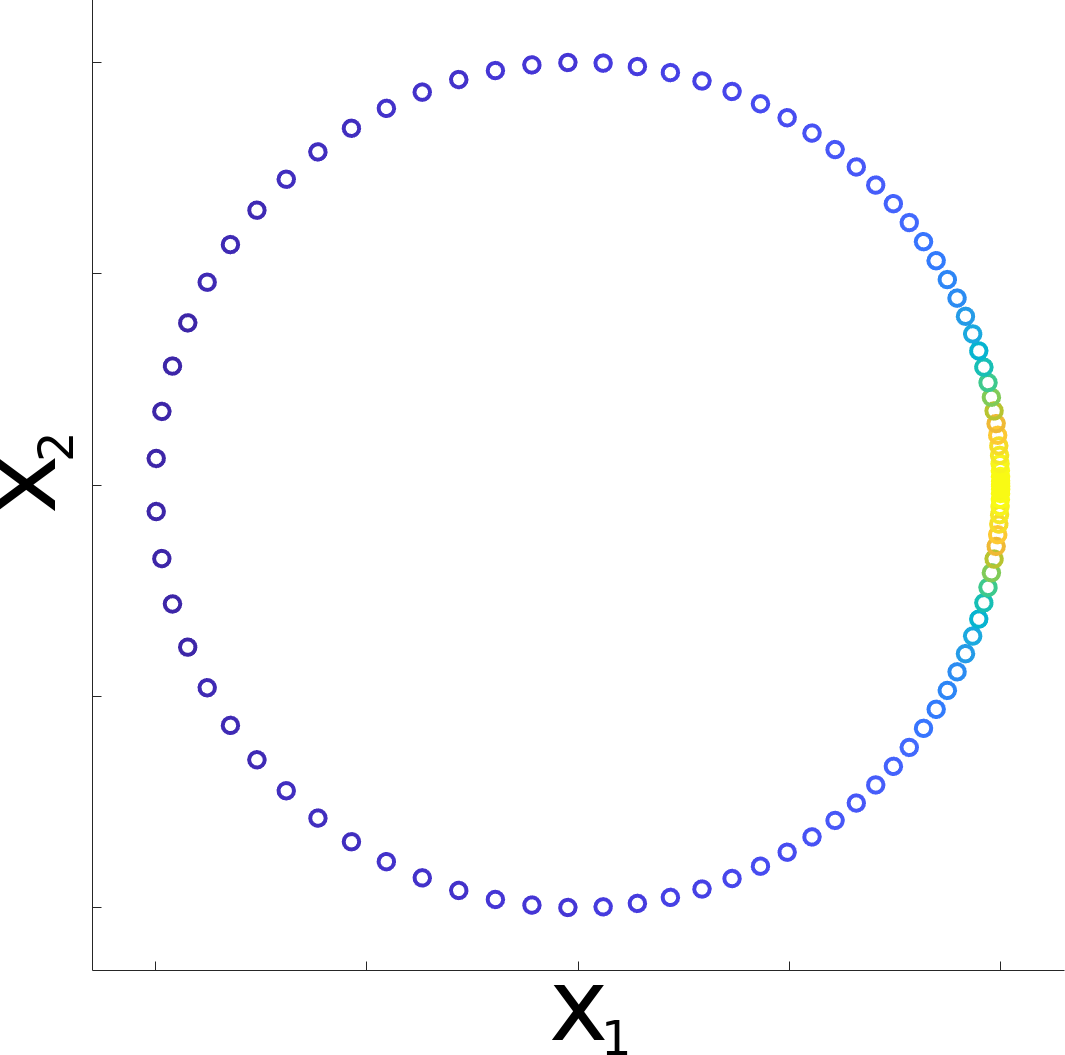}}
\subfigure[] {\label{fig:circZt0} \includegraphics[width=0.22\textwidth] {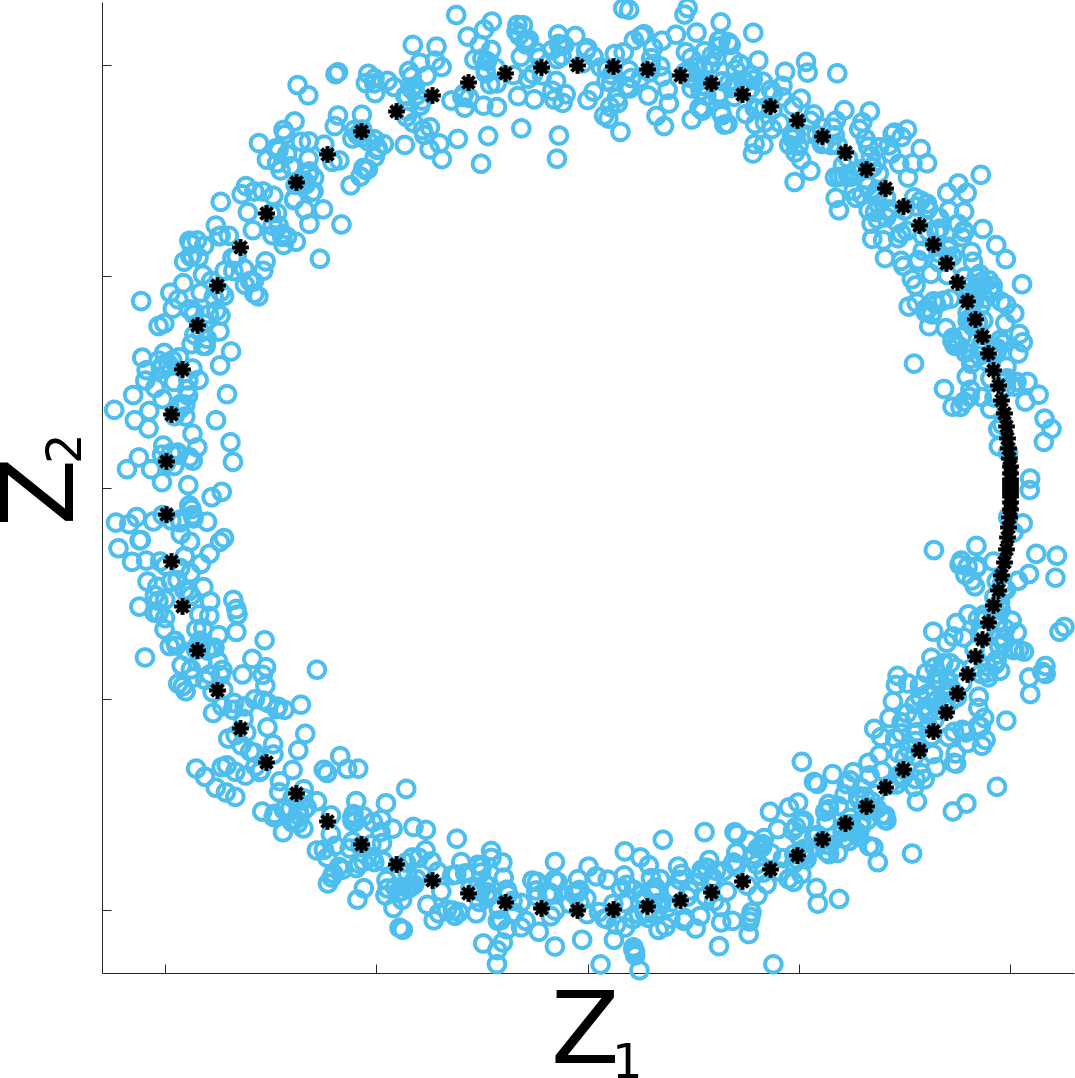}}
\subfigure[] {\label{fig:circZt1} \includegraphics[width=0.22\textwidth] {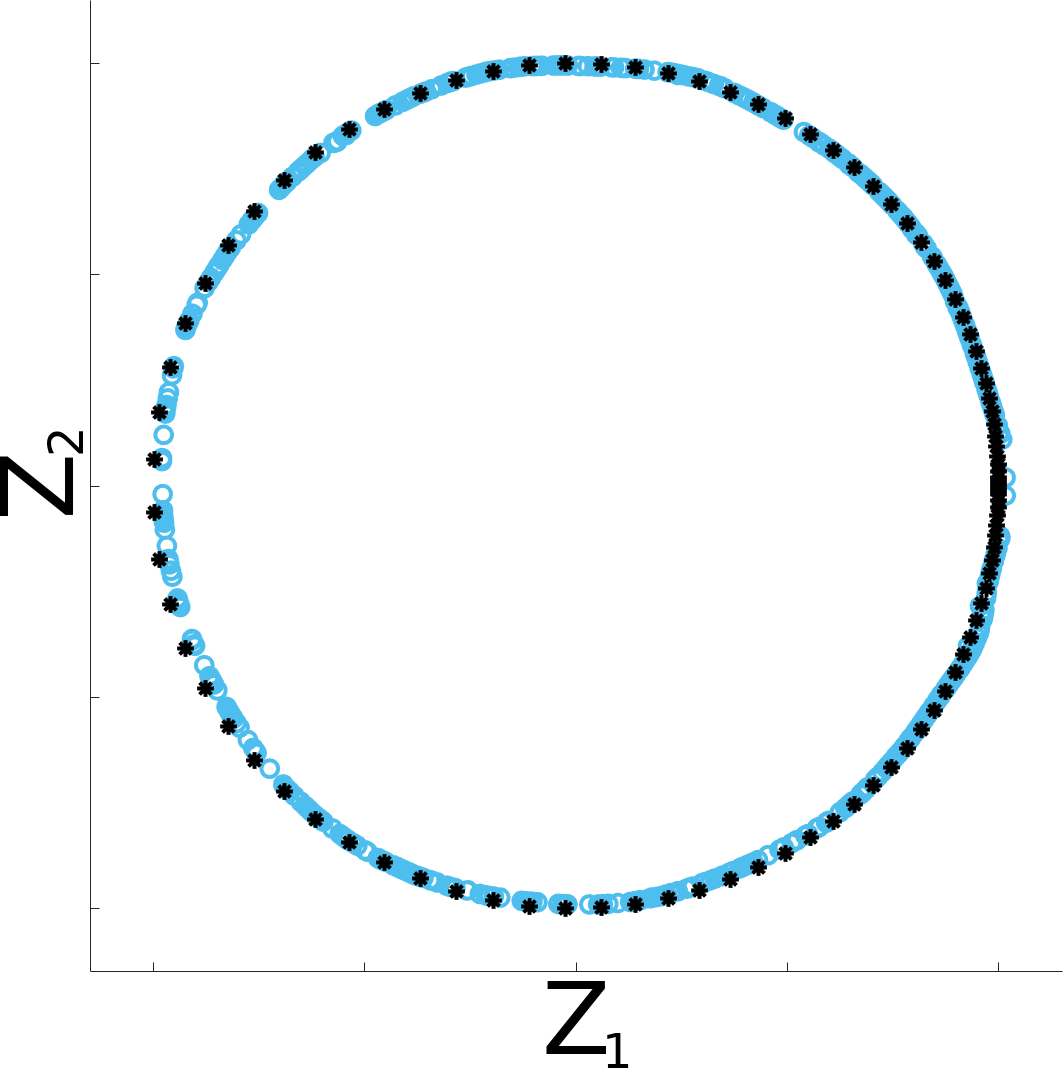}} 
\subfigure[] {\label{fig:p_val} \includegraphics[width=0.22\textwidth] {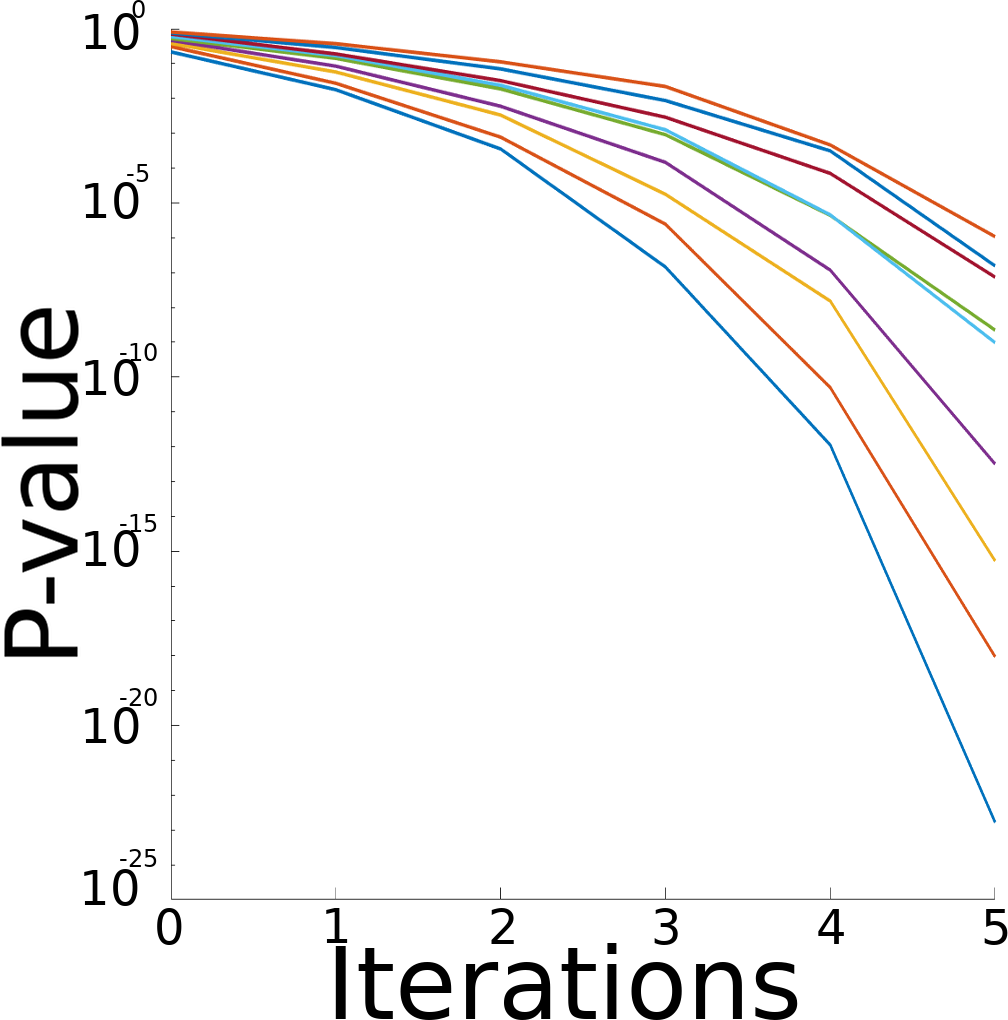}}  

\caption{Density equalization demonstrated on a circle shaped manifold. \subref{fig:circ1} The original non-uniform samples of $\myvec{X}$.  \subref{fig:circZt0} The original points $\myvec{X}$ (black asterisks) and set of new generated points $\myvec{Y}_0$ (blue circles). \subref{fig:circZt1} The final set of points $\myvec{Z}$, original points $\myvec{X}$ (black asterisks) and set of new generated points $\myvec{Y}_t$ (blue circles). In this example only one diffusion time step is required (i.e. $t=1$). \subref{fig:p_val} The p-values of a Kolmogorov-Smirnov (K-S) test comparing to a uniform distribution, the x-axis represents the number of SUGAR iterations ran.}
\end{figure}

    \color{black}
    
\subsection{Artificial Manifolds}
In the next experiment, we evaluate SUGAR on a Swiss roll, a popular manifold widely used for non-linear dimensionality reduction. We apply SUGAR to a non-uniformly generated Swiss roll. The Swiss roll $\myvec{X}$ is generated by applying the following construction steps:
\begin{itemize}
\item The random variables $\theta_i,h_i, i=1,...,N$, are drawn from the distributions ${P(\theta)},P(h)$ respectively.
\item A 3-dimensional Swiss roll is constructed based on the following function \begin{equation}{
	\label{eq:Swiss}
	\myvec{X}_i=
	\begin{bmatrix}
	{x_i}^{1}\\
	{x_i}^{2}\\
	{x_i}^{3}\\	
	\end{bmatrix}
	=
	\begin{bmatrix}
	{6\theta_i\cos(\theta_i)}\\
	{h_i}\\
	{6\theta_i\sin(\theta_i)}\\
	
	\end{bmatrix},i=1,...,N.
}
\end{equation} 
\end{itemize}
In Fig. \ref{fig:SwissNonUniform} we show a Swiss roll where $\theta_i,i=1,...,N$ is drawn from a nonuniform distribution within the range $[\frac{3\cdot \pi}{2},\frac{9\cdot \pi}{2}]$. Density of points decreases at higher values of $\theta_i$. The values $h_i,i=1,...,N$ are drawn from a uniform distribution within the range $[0,20]$ and we use $N=600$ points. 

\begin{figure}[!hbt]
\centering
\subfigure[] {\label{fig:SwissNonUniform} \includegraphics[width=0.32\textwidth] {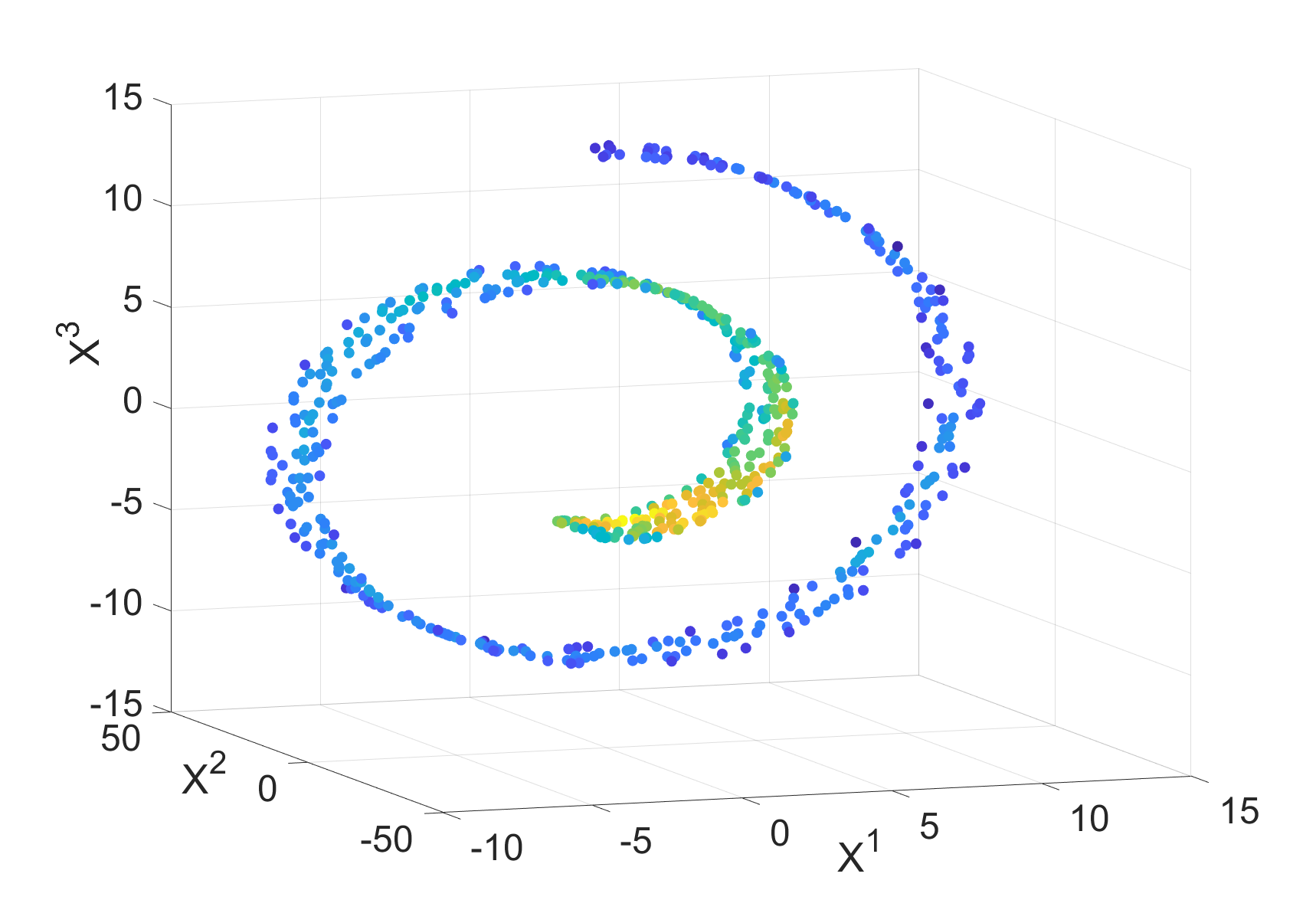}}
\subfigure[ ]{\label{fig:Magict1} \includegraphics[width=0.32\textwidth] {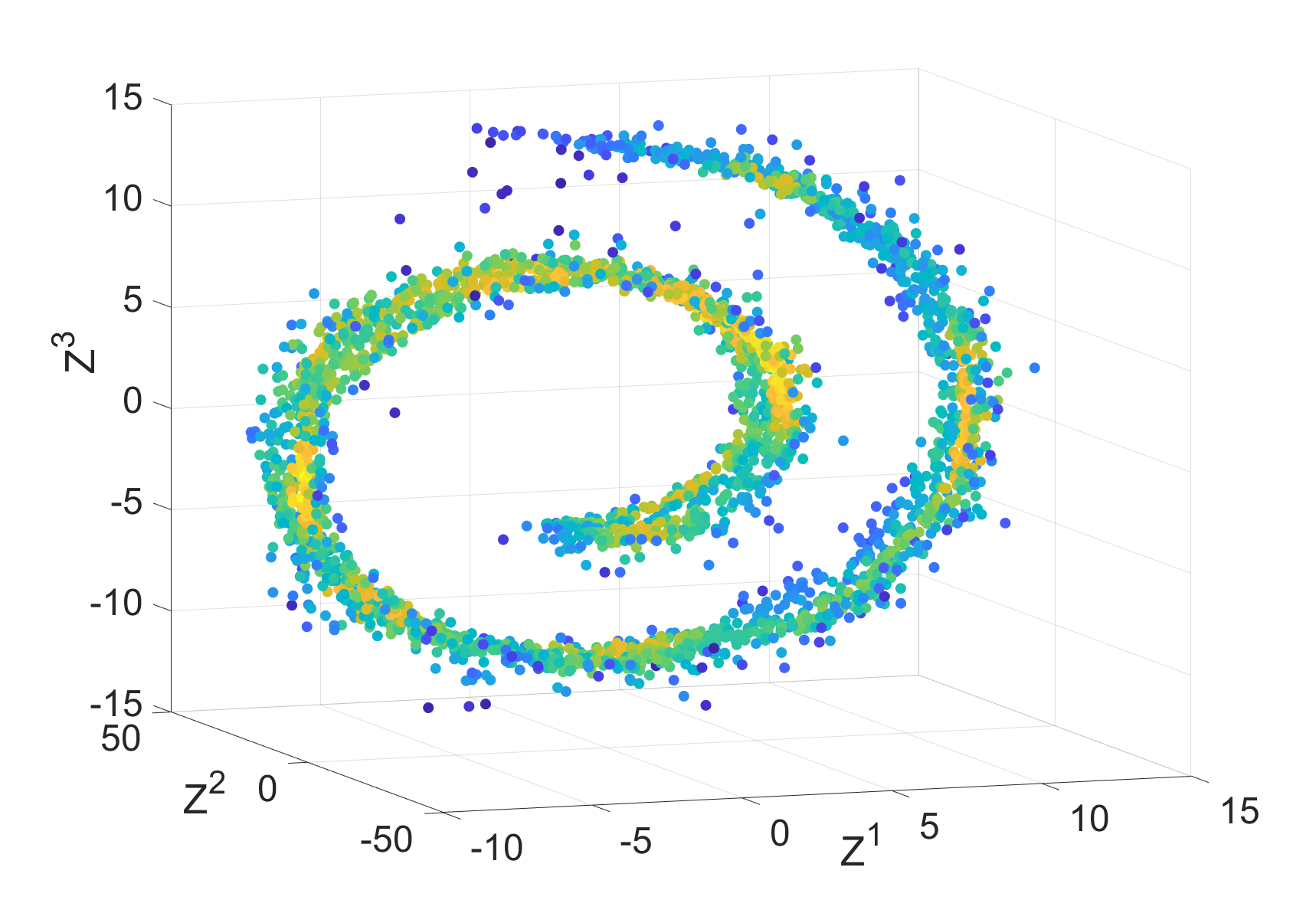}}
\subfigure[] {\label{fig:Magict10} \includegraphics[width=0.32\textwidth] {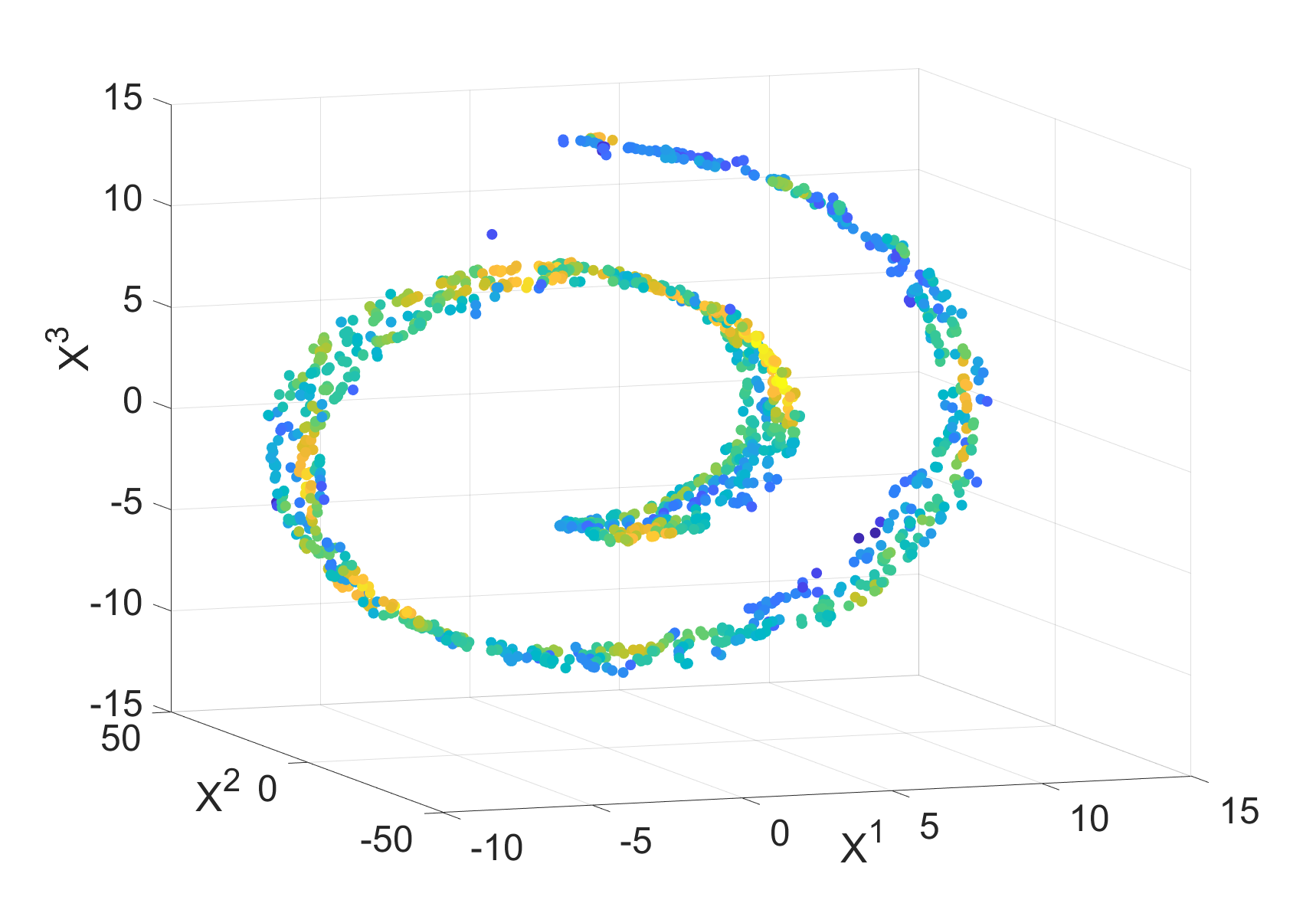}}
\subfigure[] {\label{fig:MagicCDF} \includegraphics[width=0.32\textwidth]  {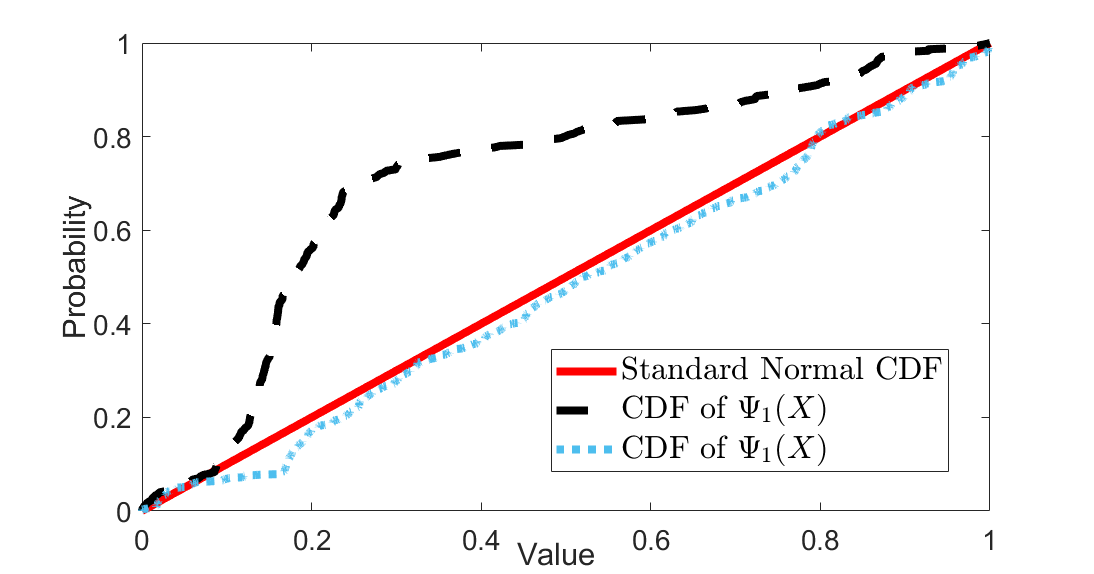}}
\caption{\subref{fig:SwissNonUniform}  A Swiss roll generated using Eq. \ref{eq:Swiss}. The variables $h_i$ are drawn from uniform distribution within the interval $ [0,100]$, while $\theta_i$ are drawn from a nonuniform distribution. Points are colored by their degree value $\hat{d}(i)$. \subref{fig:Magict1} The Swiss roll $\myvec{X}$ (black asterisk) and new generated points $\myvec{Y}_0$ (blue circle). Points are generated based on step 4 in Alg. \ref{alg:DataGeneration}. \subref{fig:Magict10} The points after applying the MGC operator $\hat{\myvec{P}}^t$ (steps 6-7 Alg. \ref{alg:DataGeneration}). Only one diffusion time step is required (i.e. $t=1$). \subref{fig:MagicCDF} An estimate of the CDF based on the first embedding (DM) coordinate of $\myvec{X} \cup \myvec{Y}$.  }
	\label{SwissMagicA}
\end{figure}
The Swiss roll is 3-dimensional but governed by one angular parameter. Thus, it allows us to evaluate whether SUGAR generates points uniformly along the Swiss roll's angular parameter. In the following, we apply Alg. \ref{alg:DataGeneration} to the Swiss roll presented in Fig. \ref{fig:SwissNonUniform}. In Fig. \ref{fig:Magict1} we present the original data $\myvec{X}$ and the generated points $\myvec{Y}_0$. The points are generated based on steps 1-4 in Alg. \ref{alg:DataGeneration}. After applying the MGC diffusion operator the points are pulled toward the Swiss roll structure. The points after applying the MGC operator $\hat{\myvec{P}}$ are presented in Fig. \ref{fig:Magict10}. In these figures the points are colored by their degree value $\hat{d}(i)$. 
In Fig. \ref{fig:MagicCDF} we present the estimated CDF before and after applying SUGAR to $\myvec{X}$. The resulted CDF resembles the CDF of a uniform distribution plotted in red. The variance of the degree $\hat{d}(i)$ drops from 0.19 to 0.09 when comparing $\myvec{X}$ and $\myvec{Z}$.

\subsection{Bunny Manifold}
In the following experiment, we evaluate the performance of SUGAR on the ``Stanford Bunny''. The ``Stanford Bunny'' is a set of points representing a 3-d surface of a bunny. We subsample the points on the surface such that the unevenly sampled set of points $\myvec{X}$ consists of 900 points. The data set $\myvec{X}$ colored by the degree estimate at each point is presented in Fig. \ref{fig:bunny_kde}. Then, we generate a set of new points $\myvec{Y}_0$ based on steps 1-4 of Algorithm \ref{alg:DataGeneration}. The unified set of points $\myvec{X} \cup \myvec{Y}_0$ is presented in Fig. \ref{fig:bunny_t0}. In the final step, we apply the MGC operator $\myvec{P}^t$, at $t=1$, the final set of points $\myvec{Z}=\myvec{X} \cup \myvec{Y}$ is presented in Fig. \ref{fig:bunny_t1}. The set $\myvec{Z}$ colored by the degree value $\hat{d}(i)$ is presented in Fig. \ref{fig:bunny_kde_z}.

\begin{figure}[!hbt]
	\centering
    \subfigure[ ]{\label{fig:bunny_kde} \includegraphics[width=0.32\textwidth] {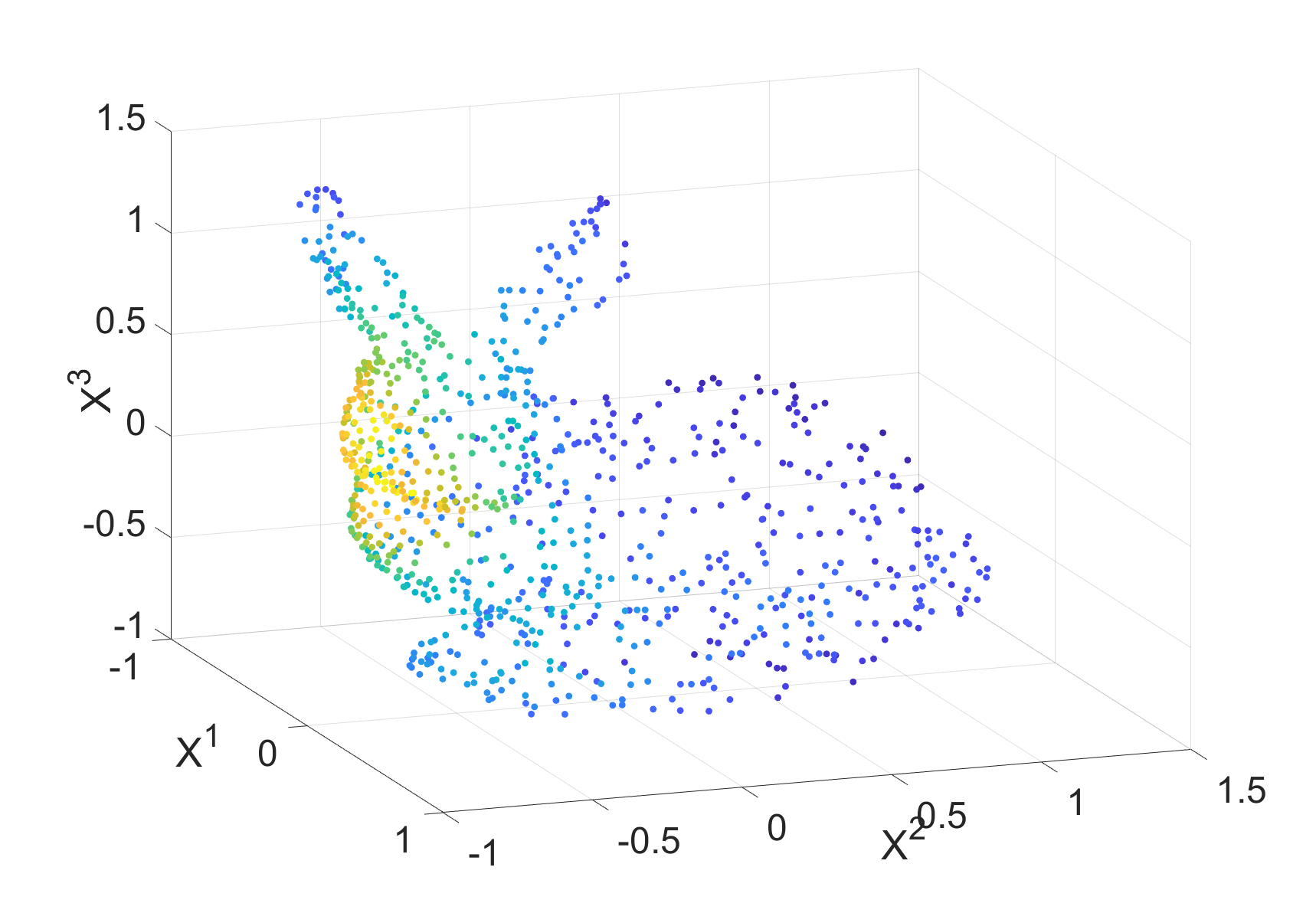}}
	\subfigure[ ]{\label{fig:bunny_t0}. \includegraphics[width=0.32\textwidth] {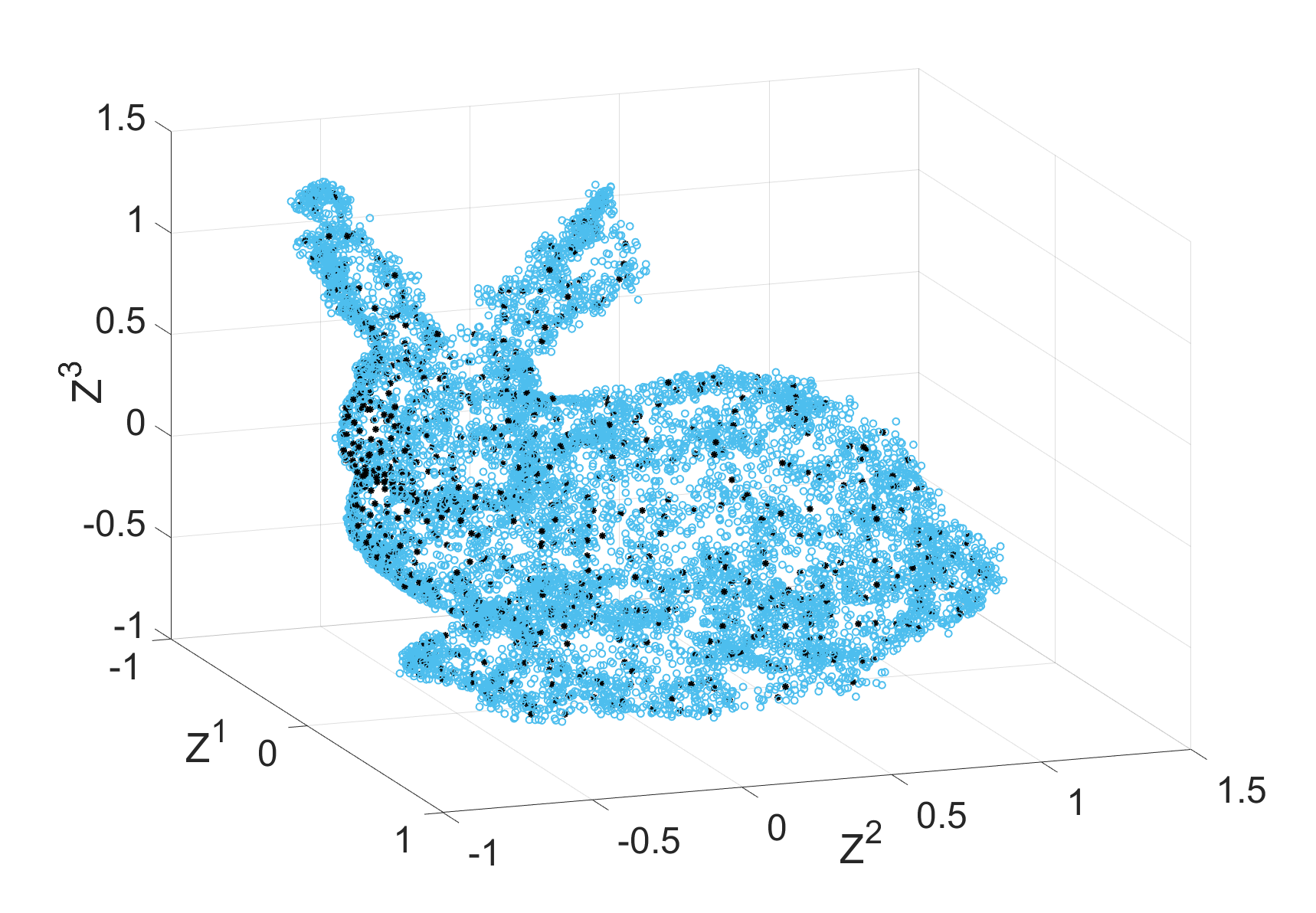}}
\subfigure[ ]{\label{fig:bunny_t1} \includegraphics[width=0.32\textwidth] {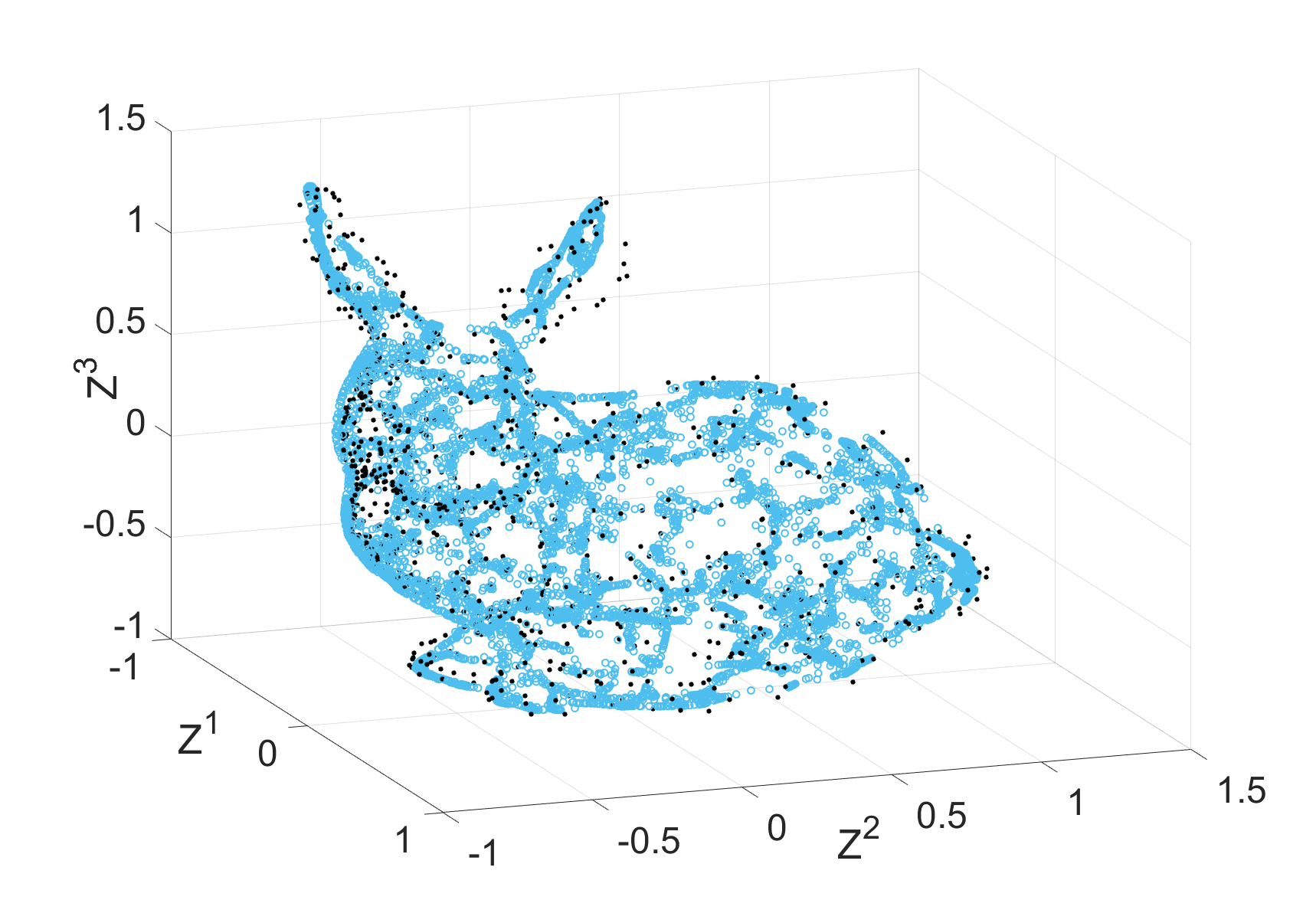}}
\subfigure[ ]{\label{fig:bunny_kde_z} \includegraphics[width=0.32\textwidth] {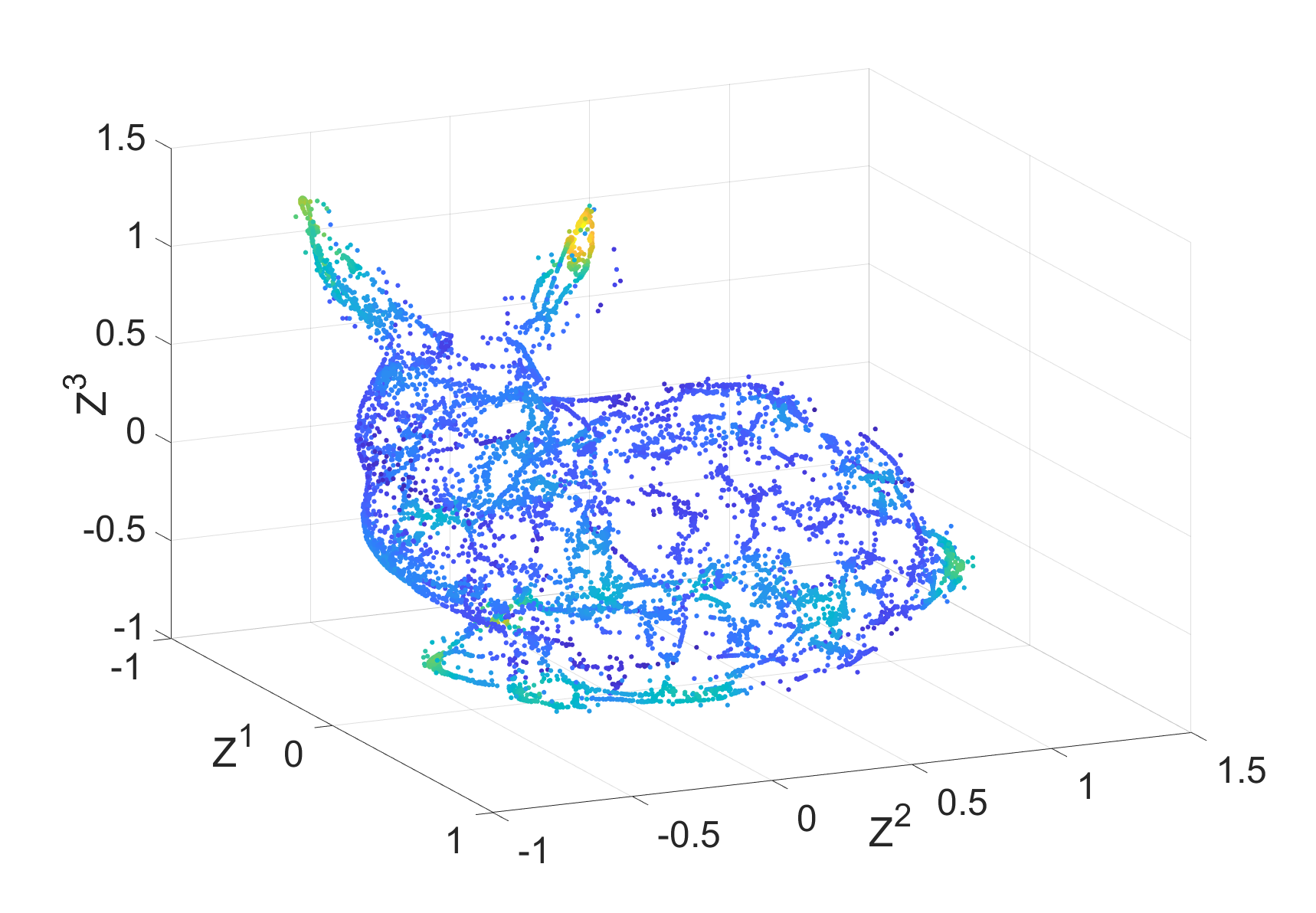}}

	\caption{\subref{fig:bunny_kde} A subset $\myvec{X}$ with 900 points from the Stanford Bunny. Points are colored by their degree value $\hat{d}(i)$. \subref{fig:bunny_t0} The original and generated points $\myvec{X} \cup \myvec{Y}_0$, black asterisks- $\myvec{X}$ original points, blue circles - new points $\myvec{Y}_0$. \subref{fig:bunny_t1} The original $\myvec{X} $ and the set $\myvec{Y}$. $\myvec{Y}$ is the set of generated points after applying the MGC diffusion operator $\hat{\myvec{P}}$ (steps 6-7 of Alg. \ref{alg:DataGeneration}). \subref{fig:bunny_kde_z} The final set again colored by the degree value $\hat{d}(i)$ at each point. }
	\label{fig:bunny}
\end{figure}
In this experiment, the output of SUGAR has preserved the complex structure of the 3-d Bunny manifold. The variance of the degree $\hat{d}(i)$ in this example, drops from 0.24 to 0.019 when comparing $\myvec{X}$ and $\myvec{Z}$.

\subsection{MNIST Manifold}
\label{sec:MnistMani}

\label{sec:mnist}
In the following experiment we empirically demonstrate ability of SUGAR to fill in missing samples and compare it to a Variational Autoencoder (VAE)~\citep{kingma2013auto} which has an implicit probabilistic model of the data.  Note that we are not able to use other density estimates in general due to the high dimensionality of datasets and the inability of density estimates to scale to high dimensions.  To begin, we rotated an example of a handwritten `6' from the MNIST dataset in $N=320$ different angles non-uniformly sampled over the range $[0,2\pi]$.  This circular construction was recovered by the diffusion maps embedding of the data, with points towards the undersampled regions having a lower degree than other regions of the embedding (Fig.~\ref{fig:Mnist}, left, colored by degree). We then generated new points around each sample in the rotated data according to Alg.~\ref{alg:DataGeneration}. We show the results of SUGAR before and after diffusion in Fig.~\ref{fig:Mnist} (top and bottom right, respectively).

\begin{figure}[!b]
	\centering
   \includegraphics[width=0.75\textwidth] {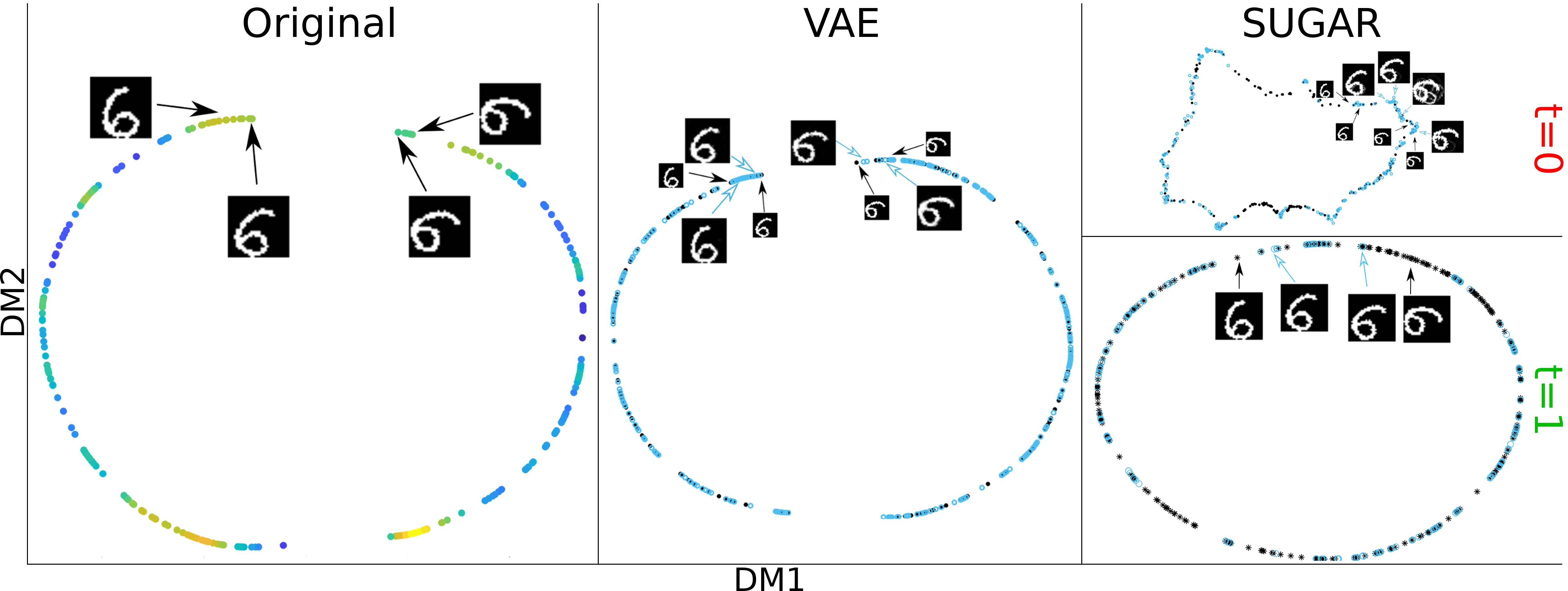}

	\caption{A 2-dimensional DM representation of: (left) Original data, 320 rotated images of handwritten '6' colored by the degree value $\hat{d}(i)$. (middle) VAE output; (right, top) SUGAR augmented data before diffusion (i.e. $t=0$); (right,bottom)  SUGAR augmented data with one step of diffusion ($t=1$).  Black asterisks - original data; Blue circles - output data. }
	\label{fig:Mnist}
\end{figure}

Next, we compared our results to a 2-layer VAE trained over the original data (Fig.~\ref{fig:Mnist}, middle).  Both SUGAR ($t=1$) and the VAE generated points along the circular structure of the original manifold. Examples images from both techniques are presented in Fig. \ref{fig:Mnist}.  Notably, the VAE generated images similar to the original angle distribution, such that sparse regions of the manifold were not filled.  In contrast, points generated by SUGAR occupied new angles not present in the original data but clearly present along the circular manifold.  This example illustrates the ability of SUGAR to recover sparse areas of a data manifold.

\subsection{Classification of Imbalanced Data}
The loss functions of many standard classification algorithms are global;  these algorithms are thus easily biased when trained on imbalanced datasets. Imbalanced training typically manifests in poor classification of rare samples. These rare samples are often important~\citep{weiss2004mining}. For example, the preponderance of healthy individuals in medical data can obscure the diagnosis of rare diseases.

Resampling and boosting strategies have been used to combat data imbalance. Removing points (undersampling) is a simple solution, but this strategy leads to information loss and can decrease generalization performance.  RUSBoost~\citep{rusboost} combines this approach with boosting, a technique that resamples the data using a set of weights learned by iterative training.  Oversampling methods remove class imbalance by generating synthetic data alongside the original data. Synthetic Minority Over-sampling Technique (SMOTE)~\citep{smote} oversamples by generating points along lines between existing points of minority classes.  

We compared SUGAR, RUSBoost, and SMOTE for improving k-NN and kernel SVM classification of 60 imbalanced data sets of varying size (from hundreds to thousands) and imbalance ratio (1.8--130) \citep{alcala2009keel}.  $\text{Precision}=\frac{TP}{TP+FP}$ and $\text{Recall}=\frac{TP}{TP+FN}$ were chosen to quantify classification performance. These metrics capture classification accuracy in light of data imbalance; precision measures the fraction of true positives to false positives, while recall measures the fraction of true positives identified. These measures are extended to multiple classes via average class precision and recall (ACP and ACR), which are defined as 
$\text{ACP}=\sum^C_{c=1}{\text{Precision}(class=c)}/C$ and $\text{ACR}=\sum^C_{c=1}{\text{Recall}(class=c)}/C$. 
These measure ignore class population biases by equally weighting classes.   This experiment is summarized in Table~\ref{table1} (see appendix for full details).

\begin{table}[!b]
  \centering
    \begin{tabular}{|l|rrrrrr|r|}
    \hline
          &       & \multicolumn{1}{l}{\textbf{Knn}} &       &       & \multicolumn{1}{l}{\textbf{SVM}} &       & \multicolumn{1}{l|}{\textbf{RusBOOST}} \\
          \hline
          & \multicolumn{1}{l}{\textbf{Orig}} & \multicolumn{1}{l}{\textbf{SMOTE}} & \multicolumn{1}{l|}{\textbf{SUGAR}} & \multicolumn{1}{l}{\textbf{Orig}} & \multicolumn{1}{l}{\textbf{SMOTE }} & \multicolumn{1}{l|}{\textbf{SUGAR}} &  \\
          \hline
    \textbf{ACP} & 0.67  & 0.76  & \textbf{0.78}  & 0.77  & 0.77  & 0.77 & 0.75 \\
    \textbf{ACR} & 0.64  & 0.73  & \textbf{0.77}  & 0.78  & 0.78  & \textbf{0.84} & 0.81 \\
    \hline

    \end{tabular}%
\caption{Average class precision (ACP) and class recall (ACR) for KNN (10-fold cross validation) and kernel SVM classifiers before and after SMOTE, SUGAR, and RusBOOST(SVM only) }
  \label{table1}%
\end{table}%


\color{black}
\subsection{Clustering of Imbalanced Data}
In order to examine the effect of SUGAR on clustering, we performed spectral clustering on a set of Gaussians in the shape of the word ``SUGAR'' (top panel, Fig.~\ref{fig:sugar_manifold}). Next, we altered the mixtures to sample heavily towards points on the edges of the word (middle panel, Fig.~\ref{fig:sugar_manifold}). This perturbation disrupted letter clusters.  Finally, we performed SUGAR on the biased data to recreate data along the manifold.  The combined data and its resultant clustering is shown in the bottom panel of Fig.~\ref{fig:sugar_manifold} revealing that the letter clustering was restored after SUGAR.

The effect of sample density on spectral clustering is evident in the eigendecomposition of the graph Laplacian, which describes the connectivity of the graph and is the basis for spectral clustering. We shall focus on the multiplicity of the zero eigenvalue, which codes for the number of connected components of a graph.  In our example, we see that the zero eigenvalue for the ground truth and SUGAR graphs has a multiplicity of $5$ whereas the corrupted graph only has a multiplicity of $4$ (see Fig.~\ref{fig:sugar_eigenvalues}).  This connectivity difference arises from the $k$-neighborhoods of points in each ground truth cluster.  We note that variation in sample density disrupts the $k$-neighborhood of points in the downsampled region to include points outside of their ground truth cluster. These connections across the letters of ``SUGAR'' thus lead to a lower multiplicity of the zero eigenvalue, which negatively affects the spectral clustering. Augmenting the biased data via SUGAR equalizes the sampling density, restoring ground-truth neighborhood structure to the graph built on the data. 
\begin{figure}[!b]
\centering
\subfigure[ ]{\label{fig:sugar_manifold} \includegraphics[width=0.26\textwidth] {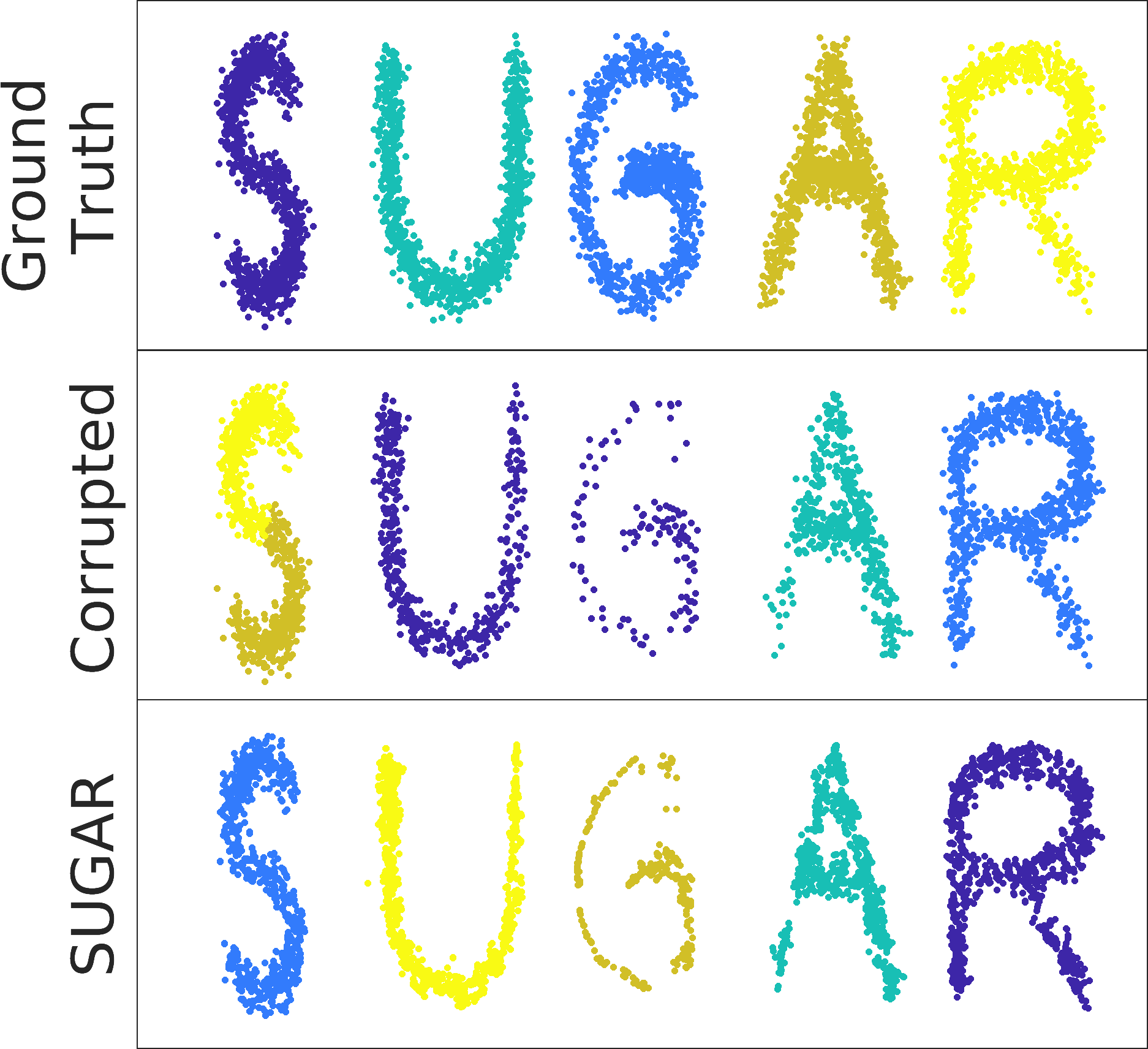}} 
\subfigure[ ]{\label{fig:sugar_eigenvalues} \includegraphics[width=0.33\textwidth] {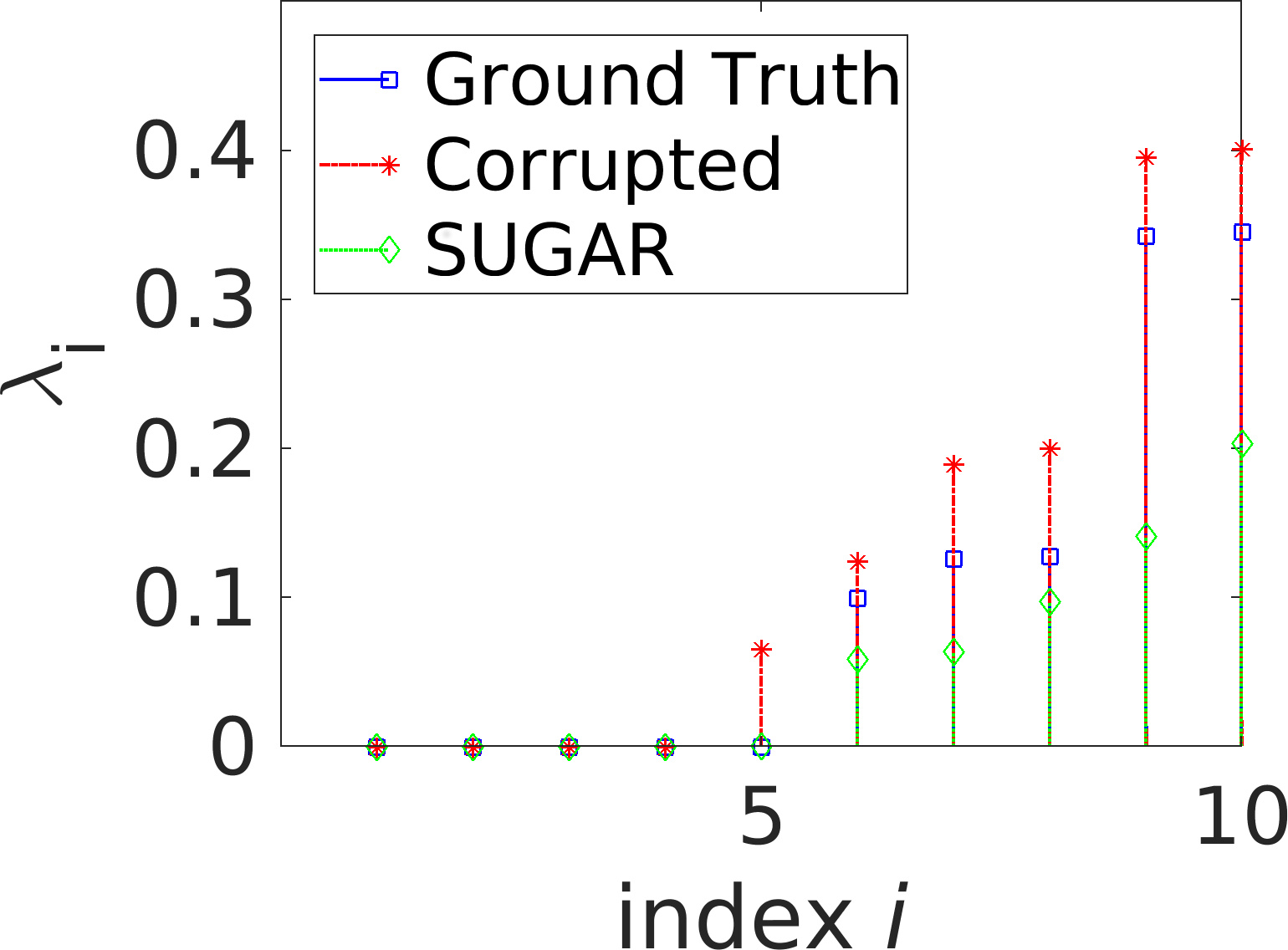}}
\subfigure[ ]{\label{fig:sugar_ri} \includegraphics[width=0.25\textwidth] {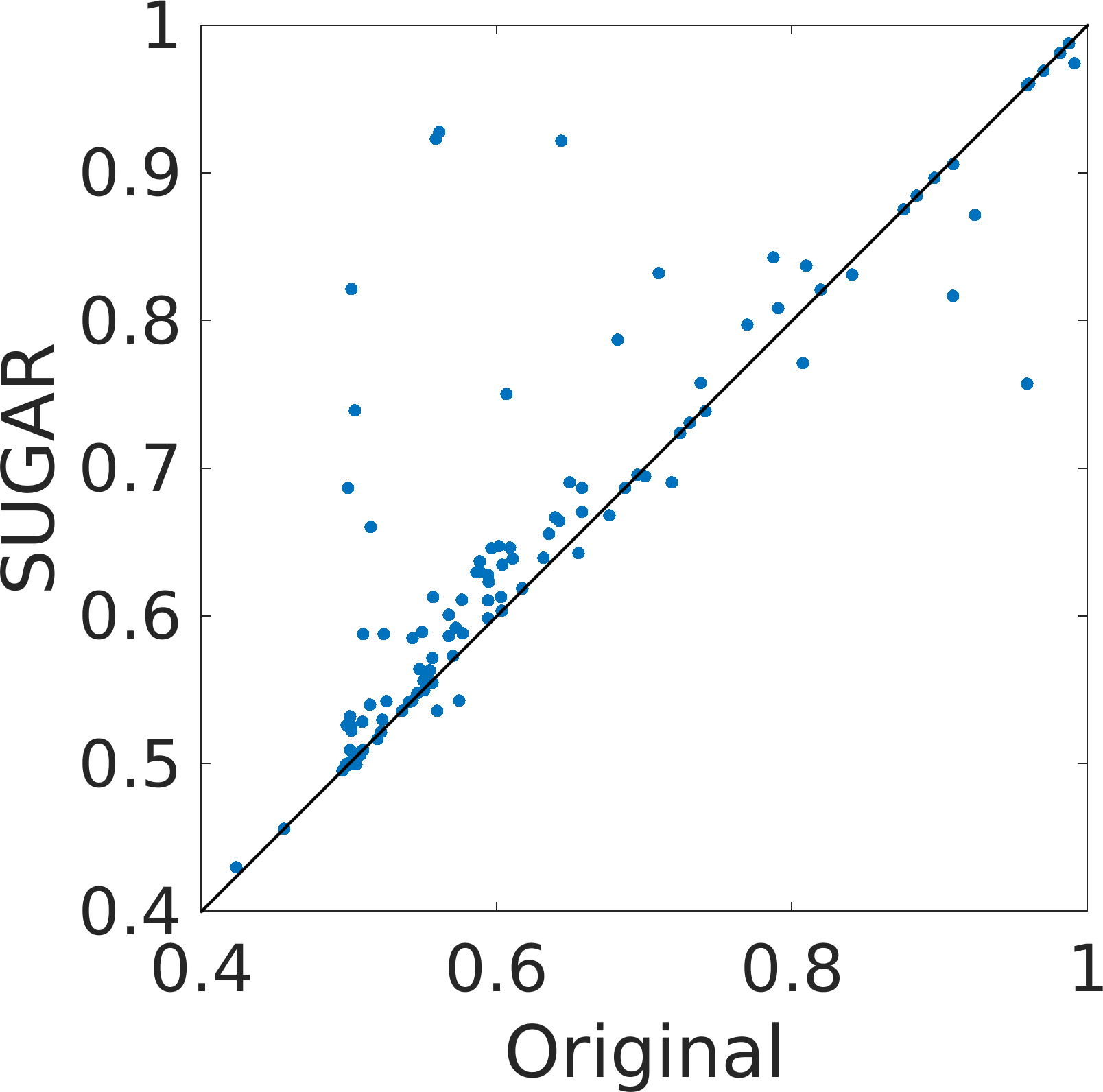}}

\caption{Augmented clustering using SUGAR. \subref{fig:sugar_manifold} Spectral clustering of a mixed Gaussian (top panel) with uneven sample density(middle panel). After SUGAR(bottom panel), the original cluster geometries are recovered.  \subref{fig:sugar_eigenvalues} Graph Laplacian eigenvalues from \subref{fig:sugar_manifold}. The corrupted graph has a lower multiplicity of the 0 eigenvalue, indicating fewer components. \subref{fig:sugar_ri} Rand Index of 115 data sets clustered by k-means before and after applying SUGAR. ~\citep{alcala2009keel}}
\end{figure}

Next, we explored the effects of SUGAR on traditional k-means across 115 datasets obtained from~\citet{alcala2009keel}. K-means was performed using the ground truth number of clusters, and the Rand Index (RI)~\citep{hubert1985comparing} between the ground truth clustering and the empirical clustering was taken (Fig.~\ref{fig:sugar_ri}, x-axis).  Subsequently, SUGAR was used to generate new points for clustering with the original data. The RI over the original data was again computed, this time using the SUGAR clusters (Fig.~\ref{fig:sugar_ri}, y-axis).  Our results indicate the SUGAR can be used to improve the cluster quality of k-means. 

\subsection{Biological Manifolds}
Next, we used SUGAR for exploratory analysis of a biological dataset.  In~\citet{velten2017human}, a high dimensional yet small ($\mathbf{X} \in \mathbb{R}^{1029\times12553}$) single cell RNA sequencing dataset was collected to elucidate the development of human blood cells, which is posited to form a continuum of development trajectories from a central reservoir of immature cells.  This dataset thus represents an ideal substrate to explore manifold learning~\citep{moon2017manifold}. However, the data presents two distinct challenge due to 1) undersampling of cell types and 2) dropout and artifacts associated with single cell RNA sequencing \citep{kim2015characterizing}.  These challenges stand at odds with a central task of computational biology, the characterization of gene-gene interactions which foment phenotypes.

We first sought to enrich rare phenotypes in the Velten data by generating $\mathbf{Y} \in \mathbb{R}^{4116\times12553}$ new points with SUGAR.  A useful tool for this analysis is the 'gene module', a pair or set of genes that are expressed together to drive phenotype development.  K-means clustering of the augmented data over 14 principal gene modules (32 dimensions) revealed 6 cell types described in~\citet{velten2017human} and a 7th cluster consisting of mature B-cells (Fig. \ref{fig:populationembedding}).  Analysis of population prevalence before and after SUGAR revealed a dramatic enrichment of mature B and pre-B cells, eosinophil/basophil/mast cells (EBM), and neutrophils (N), while previously dominant megakaryocytes (MK) became a more equal portion of the post-SUGAR population (Fig. \ref{fig:populationbar}).  These results demonstrate the ability of SUGAR to balance population prevalence along a data manifold.

In Fig. \ref{fig:bioregression}, we examine the effect of SUGAR on intra-module relationships.  Because expression of genes in a module are molecularly linked, intra-module relationships should be strong in the absence of sampling biases and experimental noise.  After SUGAR, we note an improvement in linear regression ($r^2$) and scaled mutual information coefficients.  We note that in some cases the change in mutual information was stronger than linear regression, likely due to nonlinearities in the module relationship. Because this experiment was based on putative intra-module relationships we next sought to identify strong improvements in regression coefficients \textit{de novo}.  To this end, we compared the relationship of the B cell maturation marker CD19 with the entire dataset before and after SUGAR.  In Fig. \ref{fig:biomi} we show three relationships with marked improvement from the original data~(top panel) to the augmented data~(bottom panel).  The markers uncovered by this search, HOXA3, CASP1, and EAF2, each have disparate relationships with CD19.  HOXA3 marks stem cell immaturity, and is negatively correlated with CD19.  In contrast, CASP1 is known to mark commitment to the B cell lineage~\citep{velten2017human}.  After SUGAR, both of these relationships were enhanced.  EAF2 is a part of a module that is expressed during early development of neutrophils and monocytes; we observe that its correlation and mutual information with B cell maturation are also increased after SUGAR.  We note that in light of the early development discussed by~\citet{velten2017human}, this new relationship seems problematic. In fact,~\citet{li2016eaf2} showed that EAF2 is upregulated in mature B cells as a mechanism against autoimmunity.  Taken together, our analyses show that SUGAR is effective for bolstering relationships between dimensions in the absence of prior knowledge for exploratory data analysis. 

\begin{figure}[!tb]
\centering
\subfigure[ ]{\label{fig:populationembedding} \includegraphics[width=0.29\textwidth] {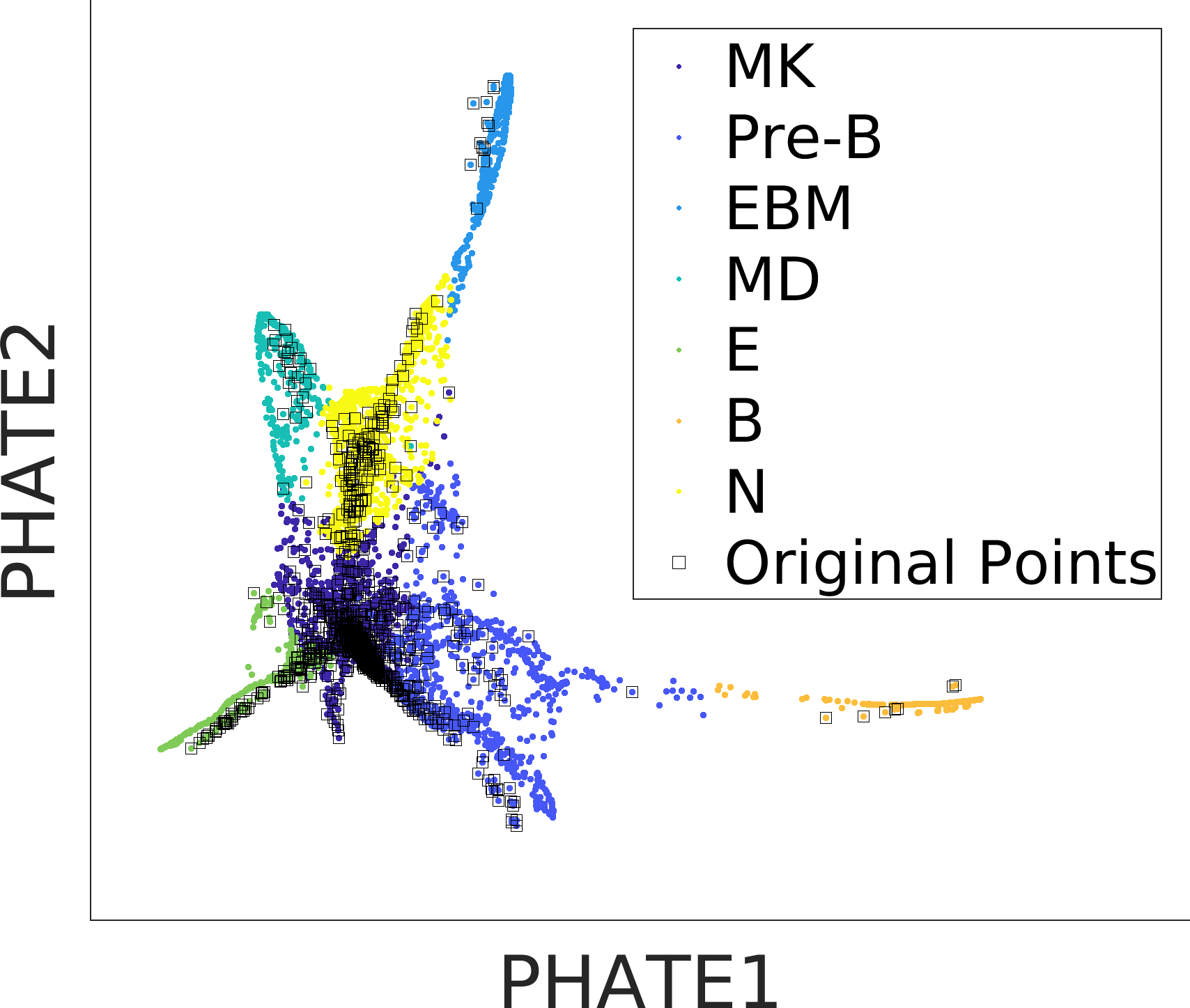}} 
\subfigure[ ]{\label{fig:populationbar} 
\includegraphics[width=0.29\textwidth] {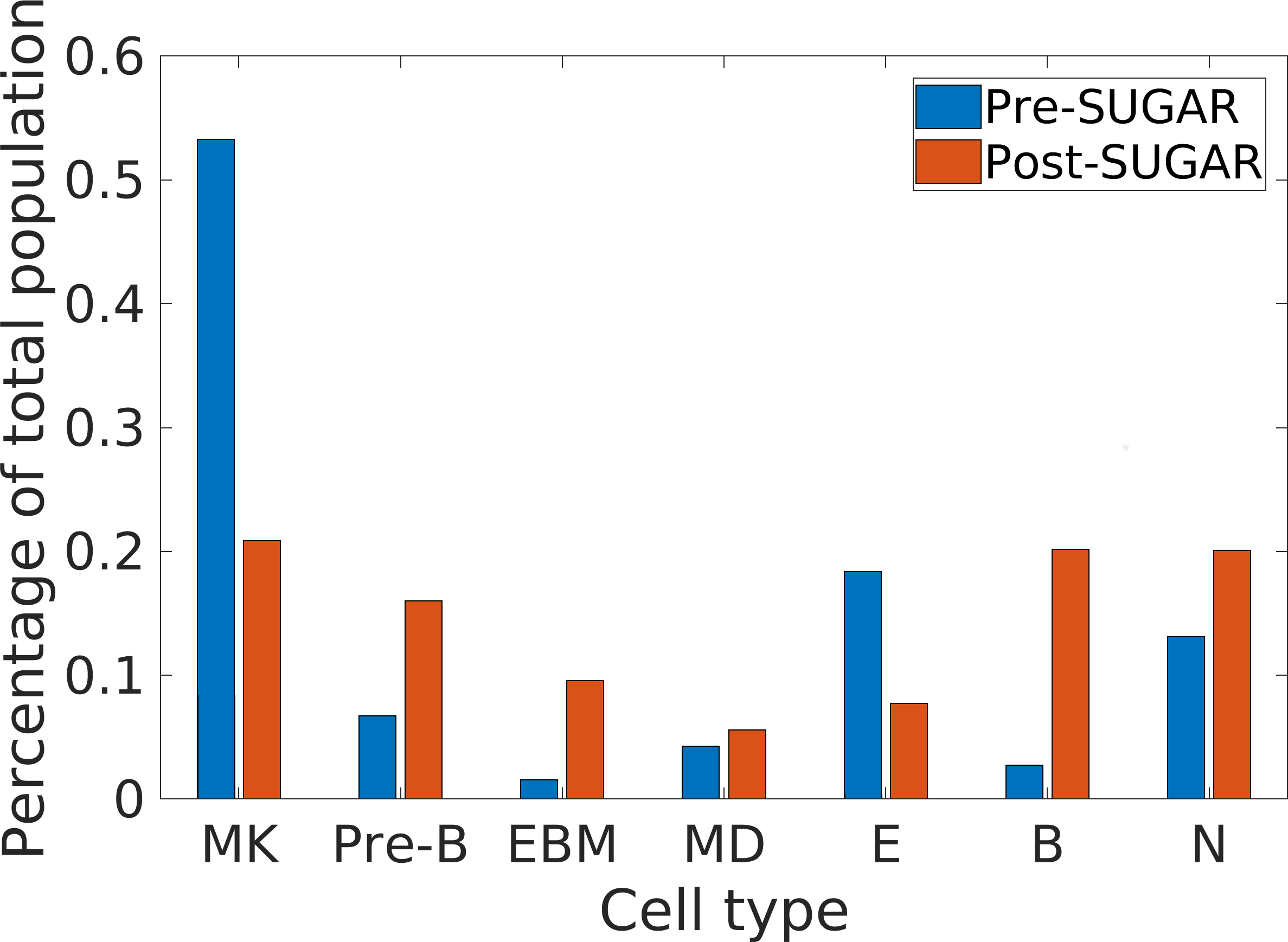}}\\
\subfigure[ ]{\label{fig:bioregression} \includegraphics[width=0.29\textwidth] {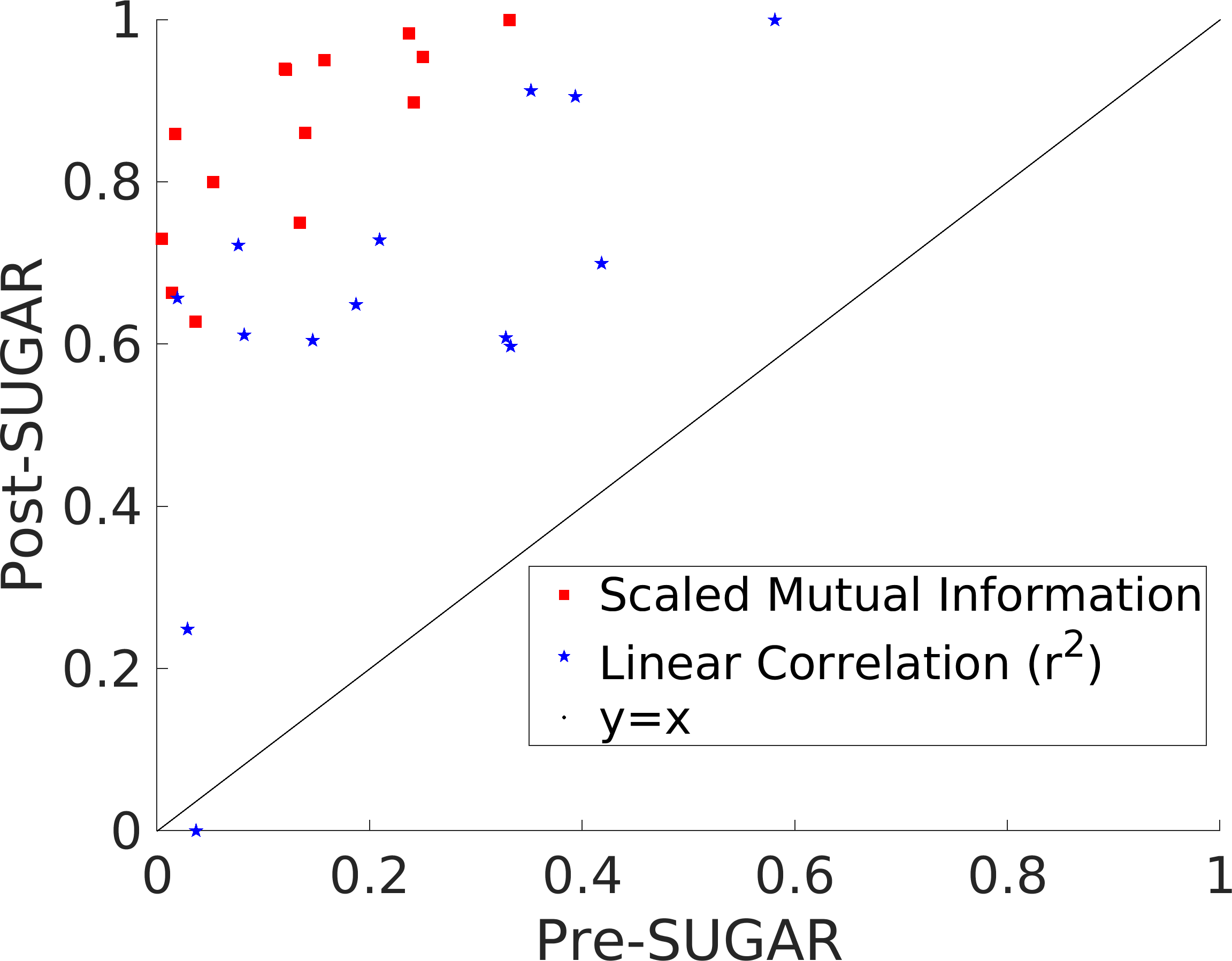}}
\subfigure[ ]{\label{fig:biomi} \includegraphics[width=0.29\textwidth] {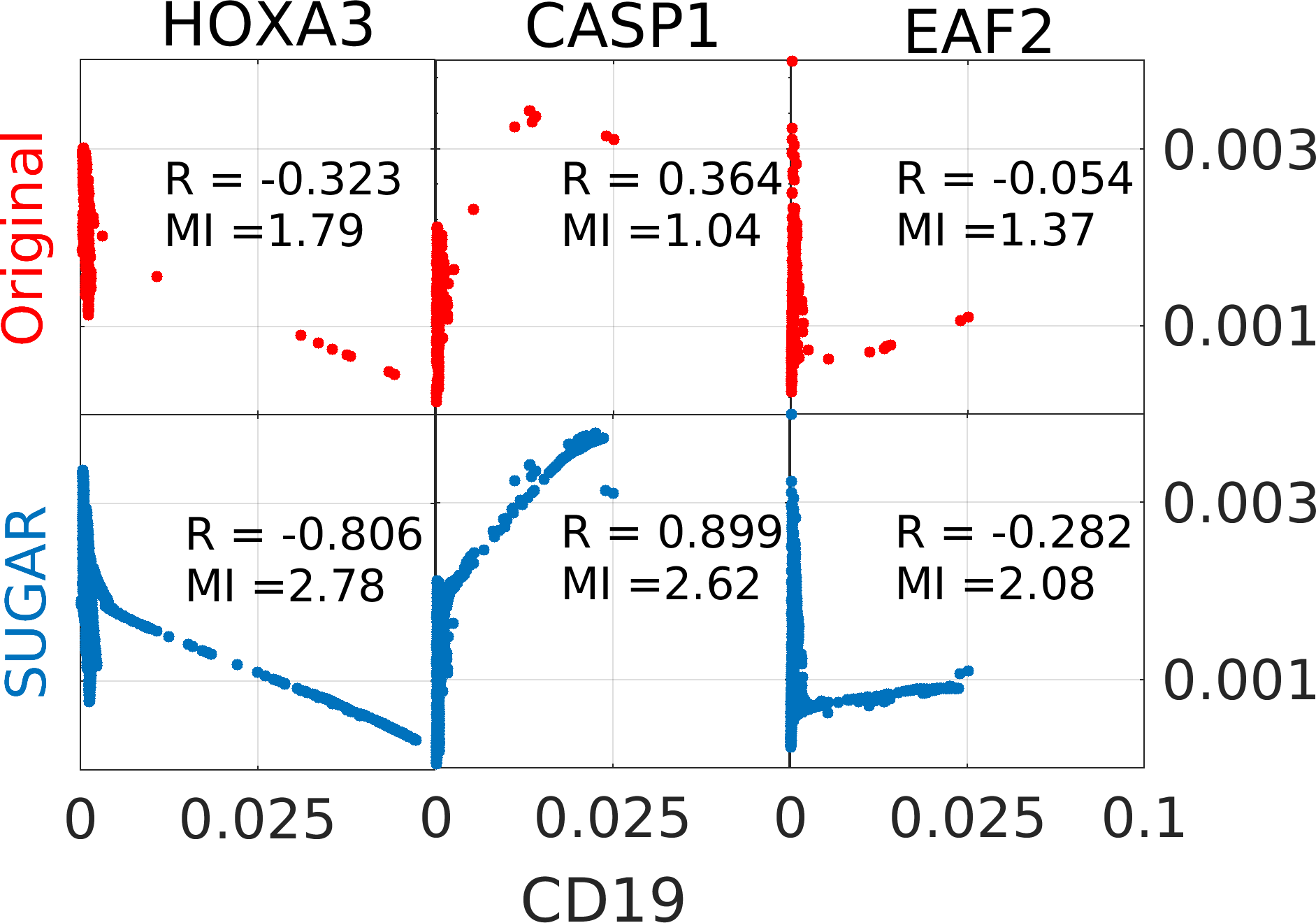}}

\caption{SUGAR was used to augment a biological dataset collected by~\citet{velten2017human}. \ref{fig:populationembedding} Augmented data embedded with PHATE \citep{moon2017visualizing} and colored by k-means over the gene module dimensions identified by~\citet{velten2017human}.  7 canonical cell types are present.  EBM: eosinophil/basophil/mast cells; N: neutrophils; MD:
monocytes/dendritic cells; E: erythroid cells; MK: megakaryocytes; Pre-B: immature B cell; B: mature B cell. \ref{fig:populationbar} Cell type prevalence before and after applying SUGAR. \ref{fig:bioregression} Explained variation (linear regression) and scaled mutual information ($\frac{\text{MI}_i}{\max{\text{MI}}}$) between the components of the 14 coexpression modules identified by~\citet{velten2017human}. \ref{fig:biomi} Relationship between CD19 (a B cell maturation marker) and HOXA3 (a cell immaturity marker), CASP1 (a B cell linear commitment marker), and EAF2 (a marker for two non-B cell types)}
\end{figure}

\section{Conclusion}
We proposed SUGAR as a new type of generative model for data, based on learning geometry rather than the density of the data. This enables us to compensate for sparsity and heavily biased sampling in many data types of interest, especially biomedical data. We assume that the training data lies on a low-dimensional manifold. The manifold assumption is usually valid in many datasets (e.g., single cell RNA sequencing~\citep{moon2017manifold}) as they are globally high-dimensional but locally generated by a small number of factors. We use a diffusion kernel to capture the manifold structure. Then, we generate new points along the incomplete manifold by randomly generating points at sparse areas. Finally, we use a weighted transition kernel to pull the new points towards the structure of the manifold. The method demonstrated promising results on synthetic data, MNIST images, and high dimensional biological datasets in applications such as clustering, classification and mutual information relationship analysis. Future work will apply SUGAR to study extremely biased biological datasets and improve classification and regression performance on them. 

\bibliography{Bib}

\def\germ{\frak} \def\scr{\cal} \ifx\documentclass\undefinedcs
  \def\bf{\fam\bffam\tenbf}\def\rm{\fam0\tenrm}\fi 
  \def\defaultdefine#1#2{\expandafter\ifx\csname#1\endcsname\relax
  \expandafter\def\csname#1\endcsname{#2}\fi} \defaultdefine{Bbb}{\bf}
  \defaultdefine{frak}{\bf} \defaultdefine{=}{\B} 
  \defaultdefine{mathfrak}{\frak} \defaultdefine{mathbb}{\bf}
  \defaultdefine{mathcal}{\cal}
  \defaultdefine{beth}{BETH}\defaultdefine{cal}{\bf} \def\bbfI{{\Bbb I}}
  \def\mbox{\hbox} \def\text{\hbox} \def\om{\omega} \def\Cal#1{{\bf #1}}
  \def\pcf{pcf} \defaultdefine{cf}{cf} \defaultdefine{reals}{{\Bbb R}}
  \defaultdefine{real}{{\Bbb R}} \def\restriction{{|}} \def\club{CLUB}
  \def\w{\omega} \def\exist{\exists} \def\se{{\germ se}} \def\bb{{\bf b}}
  \def\equivalence{\equiv} \let\lt< \let\gt>
\begin{thebibliography}{39}
\providecommand{\natexlab}[1]{#1}
\providecommand{\url}[1]{\texttt{#1}}
\expandafter\ifx\csname urlstyle\endcsname\relax
  \providecommand{\doi}[1]{doi: #1}\else
  \providecommand{\doi}{doi: \begingroup \urlstyle{rm}\Url}\fi

\bibitem[Alcal{\'a}-Fdez et~al.(2009)Alcal{\'a}-Fdez, Sanchez, Garcia, del
  Jesus, Ventura, Garrell, Otero, Romero, Bacardit, Rivas,
  et~al.]{alcala2009keel}
Alcal{\'a}-Fdez, Jes{\'u}s, Sanchez, Luciano, Garcia, Salvador, del Jesus,
  Maria~Jose, Ventura, Sebastian, Garrell, Josep~Maria, Otero, Jos{\'e},
  Romero, Crist{\'o}bal, Bacardit, Jaume, Rivas, Victor~M, et~al.
\newblock Keel: a software tool to assess evolutionary algorithms for data
  mining problems.
\newblock \emph{Soft Computing}, 13\penalty0 (3):\penalty0 307--318, 2009.

\bibitem[Bengio \& Monperrus(2005)Bengio and Monperrus]{parzen2}
Bengio, Yoshua and Monperrus, Martin.
\newblock Non-local manifold tangent learning.
\newblock In \emph{Advances in Neural Information Processing Systems}, pp.\
  129--136, 2005.

\bibitem[Bengio et~al.(2006)Bengio, Larochelle, and Vincent]{parzen3}
Bengio, Yoshua, Larochelle, Hugo, and Vincent, Pascal.
\newblock Non-local manifold parzen windows.
\newblock In \emph{Advances in neural information processing systems}, pp.\
  115--122, 2006.

\bibitem[Bermanis et~al.(2016)Bermanis, Wolf, and Averbuch]{wolf:MGC}
Bermanis, Amit, Wolf, Guy, and Averbuch, Amir.
\newblock Diffusion-based kernel methods on euclidean metric measure spaces.
\newblock \emph{Applied and Computational Harmonic Analysis}, 41\penalty0
  (1):\penalty0 190--213, 2016.

\bibitem[Bernardo et~al.(2003)Bernardo, Bayarri, Berger, Dawid, Heckerman,
  Smith, West, et~al.]{bernardo2003variational}
Bernardo, JM, Bayarri, MJ, Berger, JO, Dawid, AP, Heckerman, D, Smith, AFM,
  West, M, et~al.
\newblock The variational bayesian em algorithm for incomplete data: with
  application to scoring graphical model structures.
\newblock \emph{Bayesian statistics}, 7:\penalty0 453--464, 2003.

\bibitem[Brubaker et~al.(2012)Brubaker, Salzmann, and Urtasun]{mmcmc2}
Brubaker, Marcus, Salzmann, Mathieu, and Urtasun, Raquel.
\newblock A family of mcmc methods on implicitly defined manifolds.
\newblock In \emph{Artificial Intelligence and Statistics}, pp.\  161--172,
  2012.

\bibitem[Chawla et~al.(2002)Chawla, Bowyer, Hall, and Kegelmeyer]{smote}
Chawla, Nitesh~V, Bowyer, Kevin~W, Hall, Lawrence~O, and Kegelmeyer, W~Philip.
\newblock Smote: synthetic minority over-sampling technique.
\newblock \emph{Journal of artificial intelligence research}, 16:\penalty0
  321--357, 2002.

\bibitem[Coifman \& Lafon(2006)Coifman and Lafon]{lafon:DM}
Coifman, Ronald~R. and Lafon, Stéphane.
\newblock Diffusion maps.
\newblock \emph{Applied and Computational Harmonic Analysis}, 21\penalty0
  (1):\penalty0 5 -- 30, 2006.

\bibitem[Doersch(2016)]{doersch2016tutorial}
Doersch, Carl.
\newblock Tutorial on variational autoencoders.
\newblock \emph{arXiv preprint arXiv:1606.05908}, 2016.

\bibitem[Gin{\'e} \& Guillou(2002)Gin{\'e} and Guillou]{gine2002rates}
Gin{\'e}, Evarist and Guillou, Armelle.
\newblock Rates of strong uniform consistency for multivariate kernel density
  estimators.
\newblock In \emph{Annales de l'Institut Henri Poincare (B) Probability and
  Statistics}, volume~38, pp.\  907--921. Elsevier, 2002.

\bibitem[Girolami \& Calderhead(2011)Girolami and Calderhead]{mmcmc1}
Girolami, Mark and Calderhead, Ben.
\newblock Riemann manifold langevin and hamiltonian monte carlo methods.
\newblock \emph{Journal of the Royal Statistical Society: Series B (Statistical
  Methodology)}, 73\penalty0 (2):\penalty0 123--214, 2011.

\bibitem[Goodfellow et~al.(2014)Goodfellow, Pouget-Abadie, Mirza, Xu,
  Warde-Farley, Ozair, Courville, and Bengio]{goodfellow2014generative}
Goodfellow, Ian, Pouget-Abadie, Jean, Mirza, Mehdi, Xu, Bing, Warde-Farley,
  David, Ozair, Sherjil, Courville, Aaron, and Bengio, Yoshua.
\newblock Generative adversarial nets.
\newblock In \emph{Advances in neural information processing systems}, pp.\
  2672--2680, 2014.

\bibitem[Gr{\"u}n et~al.(2015)Gr{\"u}n, Lyubimova, Kester, Wiebrands, Basak,
  Sasaki, Clevers, and van Oudenaarden]{grun2015single}
Gr{\"u}n, Dominic, Lyubimova, Anna, Kester, Lennart, Wiebrands, Kay, Basak,
  Onur, Sasaki, Nobuo, Clevers, Hans, and van Oudenaarden, Alexander.
\newblock Single-cell messenger rna sequencing reveals rare intestinal cell
  types.
\newblock \emph{Nature}, 525\penalty0 (7568):\penalty0 251, 2015.

\bibitem[He \& Garcia(2009)He and Garcia]{he2009learning}
He, Haibo and Garcia, Edwardo~A.
\newblock Learning from imbalanced data.
\newblock \emph{IEEE Transactions on knowledge and data engineering},
  21\penalty0 (9):\penalty0 1263--1284, 2009.

\bibitem[Hensman \& Masko(2015)Hensman and Masko]{hensman2015impact}
Hensman, Paulina and Masko, David.
\newblock The impact of imbalanced training data for convolutional neural
  networks.
\newblock \emph{Degree Project in Computer Science, KTH Royal Institute of
  Technology}, 2015.

\bibitem[Hubert \& Arabie(1985)Hubert and Arabie]{hubert1985comparing}
Hubert, Lawrence and Arabie, Phipps.
\newblock Comparing partitions.
\newblock \emph{Journal of classification}, 2\penalty0 (1):\penalty0 193--218,
  1985.

\bibitem[Keller et~al.(2010)Keller, Coifman, Lafon, and Zucker]{Keller}
Keller, Yosi, Coifman, Ronald~R, Lafon, St{\'e}phane, and Zucker, Steven~W.
\newblock Audio-visual group recognition using diffusion maps.
\newblock \emph{IEEE Transactions on Signal Processing}, 58\penalty0
  (1):\penalty0 403--413, 2010.

\bibitem[Kim et~al.(2015)Kim, Kolodziejczyk, Ilicic, Teichmann, and
  Marioni]{kim2015characterizing}
Kim, Jong~Kyoung, Kolodziejczyk, Aleksandra~A, Ilicic, Tomislav, Teichmann,
  Sarah~A, and Marioni, John~C.
\newblock Characterizing noise structure in single-cell rna-seq distinguishes
  genuine from technical stochastic allelic expression.
\newblock \emph{Nature communications}, 6:\penalty0 8687, 2015.

\bibitem[Kingma \& Welling(2013)Kingma and Welling]{kingma2013auto}
Kingma, Diederik~P and Welling, Max.
\newblock Auto-encoding variational bayes.
\newblock \emph{arXiv preprint arXiv:1312.6114}, 2013.

\bibitem[Krishnaswamy et~al.(2014)Krishnaswamy, Spitzer, Mingueneau, Bendall,
  Litvin, Stone, Pe’er, and Nolan]{krishnaswamy2014conditional}
Krishnaswamy, Smita, Spitzer, Matthew~H, Mingueneau, Michael, Bendall, Sean~C,
  Litvin, Oren, Stone, Erica, Pe’er, Dana, and Nolan, Garry~P.
\newblock Conditional density-based analysis of t cell signaling in single-cell
  data.
\newblock \emph{Science}, 346\penalty0 (6213):\penalty0 1250689, 2014.

\bibitem[Li et~al.(2016)Li, Takahashi, Fujii, Zhou, Hong, Suzuki, Tsubata,
  Hase, and Wang]{li2016eaf2}
Li, Yingqian, Takahashi, Yoshimasa, Fujii, Shin-ichiro, Zhou, Yang, Hong,
  Rongjian, Suzuki, Akari, Tsubata, Takeshi, Hase, Koji, and Wang, Ji-Yang.
\newblock Eaf2 mediates germinal centre b-cell apoptosis to suppress excessive
  immune responses and prevent autoimmunity.
\newblock \emph{Nature communications}, 7:\penalty0 10836, 2016.

\bibitem[Lindenbaum et~al.(2017)Lindenbaum, Salhov, Yeredor, and
  Averbuch]{lindenbaum2017kernel}
Lindenbaum, Ofir, Salhov, Moshe, Yeredor, Arie, and Averbuch, Amir.
\newblock Kernel scaling for manifold learning and classification.
\newblock \emph{arXiv preprint arXiv:1707.01093}, 2017.

\bibitem[L{\'o}pez et~al.(2013)L{\'o}pez, Fern{\'a}ndez, Garc{\'\i}a, Palade,
  and Herrera]{lopez2013insight}
L{\'o}pez, Victoria, Fern{\'a}ndez, Alberto, Garc{\'\i}a, Salvador, Palade,
  Vasile, and Herrera, Francisco.
\newblock An insight into classification with imbalanced data: Empirical
  results and current trends on using data intrinsic characteristics.
\newblock \emph{Information Sciences}, 250:\penalty0 113--141, 2013.

\bibitem[Moon et~al.(2017{\natexlab{a}})Moon, Stanley, Burkhardt, van Dijk,
  Wolf, and Krishnaswamy]{moon2017manifold}
Moon, Kevin~R, Stanley, Jay, Burkhardt, Daniel, van Dijk, David, Wolf, Guy, and
  Krishnaswamy, Smita.
\newblock Manifold learning-based methods for analyzing single-cell
  rna-sequencing data.
\newblock \emph{Current Opinion in Systems Biology}, 2017{\natexlab{a}}.

\bibitem[Moon et~al.(2017{\natexlab{b}})Moon, van Dijk, Wang, Burkhardt, Chen,
  van~den Elzen, Hirn, Coifman, Ivanova, Wolf, et~al.]{moon2017visualizing}
Moon, Kevin~R, van Dijk, David, Wang, Zheng, Burkhardt, Daniel, Chen, William,
  van~den Elzen, Antonia, Hirn, Matthew~J, Coifman, Ronald~R, Ivanova,
  Natalia~B, Wolf, Guy, et~al.
\newblock Visualizing transitions and structure for high dimensional data
  exploration.
\newblock \emph{bioRxiv}, pp.\  120378, 2017{\natexlab{b}}.

\bibitem[{\"O}ztireli et~al.(2010){\"O}ztireli, Alexa, and
  Gross]{oztireli2010spectral}
{\"O}ztireli, A~Cengiz, Alexa, Marc, and Gross, Markus.
\newblock Spectral sampling of manifolds.
\newblock In \emph{ACM Transactions on Graphics (TOG)}, volume~29, pp.\  168.
  ACM, 2010.

\bibitem[Rasmussen(2000)]{rasmussen2000infinite}
Rasmussen, Carl~Edward.
\newblock The infinite gaussian mixture model.
\newblock In \emph{Advances in neural information processing systems}, pp.\
  554--560, 2000.

\bibitem[Scott(1985)]{scott1985averaged}
Scott, David~W.
\newblock Averaged shifted histograms: effective nonparametric density
  estimators in several dimensions.
\newblock \emph{The Annals of Statistics}, pp.\  1024--1040, 1985.

\bibitem[Scott(2008)]{scott2008kernel}
Scott, David~W.
\newblock Kernel density estimators.
\newblock \emph{Multivariate Density Estimation: Theory, Practice, and
  Visualization}, pp.\  125--193, 2008.

\bibitem[Seaman \& Powell(1996)Seaman and Powell]{seaman1996evaluation}
Seaman, D~Erran and Powell, Roger~A.
\newblock An evaluation of the accuracy of kernel density estimators for home
  range analysis.
\newblock \emph{Ecology}, 77\penalty0 (7):\penalty0 2075--2085, 1996.

\bibitem[Seiffert et~al.(2010)Seiffert, Khoshgoftaar, Van~Hulse, and
  Napolitano]{rusboost}
Seiffert, Chris, Khoshgoftaar, Taghi~M, Van~Hulse, Jason, and Napolitano, Amri.
\newblock Rusboost: A hybrid approach to alleviating class imbalance.
\newblock \emph{IEEE Transactions on Systems, Man, and Cybernetics-Part A:
  Systems and Humans}, 40\penalty0 (1):\penalty0 185--197, 2010.

\bibitem[Singer et~al.(2009)Singer, Erban, Kevrekidis, and
  Coifman]{singer2009detecting}
Singer, Amit, Erban, Radek, Kevrekidis, Ioannis~G, and Coifman, Ronald~R.
\newblock Detecting intrinsic slow variables in stochastic dynamical systems by
  anisotropic diffusion maps.
\newblock \emph{Proceedings of the National Academy of Sciences}, 106\penalty0
  (38):\penalty0 16090--16095, 2009.

\bibitem[Varanasi \& Aazhang(1989)Varanasi and Aazhang]{varanasi1989parametric}
Varanasi, Mahesh~K and Aazhang, Behnaam.
\newblock Parametric generalized gaussian density estimation.
\newblock \emph{The Journal of the Acoustical Society of America}, 86\penalty0
  (4):\penalty0 1404--1415, 1989.

\bibitem[Velten et~al.(2017)Velten, Haas, Raffel, Blaszkiewicz, Islam, Hennig,
  Hirche, Lutz, Buss, Nowak, et~al.]{velten2017human}
Velten, Lars, Haas, Simon~F, Raffel, Simon, Blaszkiewicz, Sandra, Islam,
  Saiful, Hennig, Bianca~P, Hirche, Christoph, Lutz, Christoph, Buss, Eike~C,
  Nowak, Daniel, et~al.
\newblock Human haematopoietic stem cell lineage commitment is a continuous
  process.
\newblock \emph{Nat. Cell Biol}, 19:\penalty0 271--281, 2017.

\bibitem[Vincent \& Bengio(2003)Vincent and Bengio]{parzen1}
Vincent, Pascal and Bengio, Yoshua.
\newblock Manifold parzen windows.
\newblock In \emph{Advances in neural information processing systems}, pp.\
  849--856, 2003.

\bibitem[Weiss(2004)]{weiss2004mining}
Weiss, Gary~M.
\newblock Mining with rarity: a unifying framework.
\newblock \emph{ACM Sigkdd Explorations Newsletter}, 6\penalty0 (1):\penalty0
  7--19, 2004.

\bibitem[Wu(2012)]{wu2012uniform}
Wu, Junjie.
\newblock The uniform effect of k-means clustering.
\newblock In \emph{Advances in K-means Clustering}, pp.\  17--35. Springer,
  2012.

\bibitem[Xuan et~al.(2013)Xuan, Zhigang, and Fan]{xuan2013exploring}
Xuan, Li, Zhigang, Chen, and Fan, Yang.
\newblock Exploring of clustering algorithm on class-imbalanced data.
\newblock In \emph{Computer Science \& Education (ICCSE), 2013 8th
  International Conference on}, pp.\  89--93. IEEE, 2013.

\bibitem[Zelnik-Manor \& Perona(2005)Zelnik-Manor and Perona]{GKernelZelnik}
Zelnik-Manor, Lihi and Perona, Pietro.
\newblock Self-tuning spectral clustering.
\newblock In \emph{Advances in neural information processing systems}, pp.\
  1601--1608, 2005.

\end{thebibliography}
\bibliographystyle{icml2018}

\section{Appendix}
	\begin{longtable}{lrrrrrrr}\label{table2}
		&       &       & \multicolumn{1}{l}{\textbf{Recall}} &       &       &       &  \\
		&       & \multicolumn{1}{l}{\textbf{Knn}} &       &       & \multicolumn{1}{l}{\textbf{SVM}} &       &  \\
		& \multicolumn{1}{l}{\textbf{Orig}} & \multicolumn{1}{l}{\textbf{SMOTE}} & \multicolumn{1}{l}{\textbf{SUGAR}} & \multicolumn{1}{l}{\textbf{Orig}} & \multicolumn{1}{l}{\textbf{SMOTE }} & \multicolumn{1}{l}{\textbf{SUGAR}} & \multicolumn{1}{l}{\textbf{RusBOOST}} \\
		\textbf{ecoli-013} & 0.50  & 0.67  & \textbf{0.72} & 0.63  & 0.64  & 0.70  & 0.66 \\
		\textbf{ecoli-013} & 0.61  & 0.66  & \textbf{0.71} & 0.65  & 0.64  & \textbf{0.71} & 0.66 \\
		\textbf{ecoli1} & 0.57  & 0.57  & \textbf{0.66} & 0.58  & 0.57  & 0.65  & 0.57 \\
		\textbf{ecoli2} & 0.51  & 0.56  & 0.63  & 0.62  & 0.61  & \textbf{0.68} & 0.54 \\
		\textbf{ecoli3} & 0.80  & 0.73  & 0.79  & 0.85  & 0.84  & \textbf{0.89} & 0.85 \\
		\textbf{glass-0123} & 0.73  & 0.89  & 0.90  & 0.94  & 0.94  & \textbf{0.95} & 0.92 \\
		\textbf{glass-0162} & 0.50  & 0.53  & 0.60  & 0.69  & 0.70  & 0.67  & \textbf{0.74} \\
		\textbf{glass-0165} & 0.50  & 0.64  & 0.70  & 0.69  & 0.69  & \textbf{0.80} & 0.75 \\
		\textbf{glass0} & 0.72  & 0.73  & 0.80  & 0.83  & 0.83  & \textbf{0.80} & 0.83 \\
		\textbf{glass1} & 0.61  & 0.74  & 0.70  & 0.74  & 0.76  & \textbf{0.80} & 0.77 \\
		\textbf{glass2} & 0.58  & 0.63  & 0.64  & 0.72  & \textbf{0.74} & 0.69  & \textbf{0.74} \\
		\textbf{glass4} & 0.52  & 0.78  & 0.78  & 0.71  & 0.71  & \textbf{0.87} & 0.84 \\
		\textbf{glass5} & 0.54  & 0.68  & 0.71  & 0.67  & 0.67  & \textbf{0.76} & 0.74 \\
		\textbf{glass6} & 0.83  & 0.86  & 0.93  & 0.82  & 0.82  & \textbf{0.93} & 0.87 \\
		\textbf{haberman} & 0.52  & 0.59  & 0.63  & 0.62  & 0.59  & \textbf{0.66} & 0.59 \\
		\textbf{new-thyroid1} & 0.88  & 0.91  & 0.95  & \textbf{0.99} & \textbf{0.99} & \textbf{0.99} & 0.95 \\
		\textbf{new-thyroid2} & 0.75  & 0.85  & 0.93  & \textbf{0.99} & \textbf{0.99} & \textbf{0.99} & 0.95 \\
		\textbf{pb-134} & 0.61  & 0.82  & 0.94  & 0.78  & 0.91  & 0.93  & \textbf{0.95} \\
		\textbf{pb-0} & 0.85  & 0.88  & 0.90  & \textbf{0.95} & \textbf{0.95} & \textbf{0.95} & 0.94 \\
		\textbf{pima} & 0.63  & 0.63  & 0.66  & 0.73  & 0.71  & \textbf{0.74} & 0.71 \\
		\textbf{segment0} & 0.96  & 0.99  & 0.99  & 0.99  & \textbf{1.00} & \textbf{1.00} & 0.99 \\
		\textbf{shuttle-c24} & 0.70  & 0.70  & 0.70  & 0.65  & 0.68  & \textbf{0.70} & 0.50 \\
		\textbf{vehicle0} & 0.86  & 0.92  & 0.92  & 0.92  & 0.91  & \textbf{0.97} & 0.95 \\
		\textbf{vehicle1} & 0.56  & 0.57  & 0.62  & 0.68  & 0.69  & \textbf{0.85} & 0.67 \\
		\textbf{vehicle2} & 0.78  & 0.84  & 0.86  & 0.94  & 0.92  & \textbf{0.97} & 0.93 \\
		\textbf{vehicle3} & 0.56  & 0.56  & 0.61  & 0.67  & 0.69  & \textbf{0.85} & 0.68 \\
		\textbf{vowel0} & 0.90  & 0.99  & 0.99  & 0.98  & \textbf{1.00} & \textbf{1.00} & 0.98 \\
		\textbf{wisconsin} & 0.95  & 0.96  & 0.96  & \textbf{0.97} & \textbf{0.97} & \textbf{0.97} & 0.96 \\
		\textbf{yeast-0567} & 0.51  & 0.60  & 0.65  & 0.74  & 0.73  & 0.79  & \textbf{0.80} \\
		\textbf{yeast-1289} & 0.50  & 0.50  & 0.66  & 0.60  & 0.57  & 0.69  & \textbf{0.71} \\
		\textbf{yeast-1458} & 0.50  & 0.50  & 0.57  & 0.61  & 0.59  & 0.60  & \textbf{0.65} \\
		\textbf{yeast-17} & 0.50  & 0.54  & 0.65  & 0.64  & 0.63  & 0.71  & \textbf{0.79} \\
		\textbf{yeast-24} & 0.62  & 0.77  & 0.78  & 0.81  & 0.83  & \textbf{0.89} & 0.85 \\
		\textbf{yeast-28} & 0.63  & 0.68  & 0.73  & 0.61  & 0.57  & \textbf{0.74} & 0.72 \\
		\textbf{yeast1} & 0.62  & 0.65  & 0.68  & 0.71  & 0.70  & 0.71  & \textbf{0.72} \\
		\textbf{yeast3} & 0.81  & 0.81  & 0.90  & 0.88  & 0.87  & \textbf{0.92} & 0.87 \\
		\textbf{yeast4} & 0.63  & 0.66  & 0.71  & 0.72  & 0.73  & 0.81  & \textbf{0.84} \\
		\textbf{yeast5} & 0.71  & 0.82  & 0.87  & 0.90  & 0.91  & \textbf{0.95} & 0.88 \\
		\textbf{yeast6} & 0.70  & 0.78  & 0.79  & 0.80  & 0.79  & \textbf{0.86} & 0.84 \\
		\textbf{ecoli-0145} & 0.73  & 0.83  & 0.90  & 0.84  & 0.84  & \textbf{0.93} & 0.84 \\
		\textbf{ecoli-0-147} & 0.56  & 0.77  & 0.77  & 0.80  & 0.81  & \textbf{0.85} & \textbf{0.85} \\
		\textbf{ecoli-01475} & 0.54  & 0.83  & 0.85  & 0.87  & 0.88  & \textbf{0.92} & \textbf{0.92} \\
		\textbf{ecoli-0123} & 0.74  & 0.82  & 0.88  & 0.80  & 0.82  & \textbf{0.92} & 0.88 \\
		\textbf{ecoli-015} & 0.76  & 0.82  & 0.90  & 0.83  & 0.83  & \textbf{0.92} & 0.85 \\
		\textbf{ecoli-0234} & 0.74  & 0.81  & 0.87  & 0.83  & 0.83  & \textbf{0.92} & 0.85 \\
		\textbf{ecoli-0267} & 0.57  & 0.76  & 0.75  & 0.79  & 0.78  & \textbf{0.83} & \textbf{0.83} \\
		\textbf{ecoli-0346} & 0.77  & 0.83  & 0.88  & 0.84  & 0.84  & \textbf{0.94} & 0.85 \\
		\textbf{ecoli-0347} & 0.59  & 0.83  & 0.81  & 0.86  & 0.87  & \textbf{0.89} & 0.88 \\
		\textbf{ecoli-0345} & 0.71  & 0.80  & 0.87  & 0.82  & 0.82  & \textbf{0.91} & 0.82 \\
		\textbf{ecoli-0465} & 0.74  & 0.82  & 0.88  & 0.84  & 0.85  & \textbf{0.93} & 0.86 \\
		\textbf{ecoli-0673} & 0.60  & 0.74  & 0.78  & 0.79  & 0.80  & \textbf{0.87} & 0.81 \\
		\textbf{ecoli-0675} & 0.58  & 0.75  & 0.79  & 0.82  & 0.81  & \textbf{0.89} & 0.85 \\
		\textbf{glass-0146} & 0.50  & 0.51  & 0.57  & 0.69  & 0.70  & 0.67  & \textbf{0.72} \\
		\textbf{glass-015} & 0.50  & 0.50  & 0.54  & 0.66  & 0.66  & 0.68  & \textbf{0.70} \\
		\textbf{glass-045} & 0.50  & 0.73  & 0.80  & 0.70  & 0.70  & 0.83  & \textbf{0.84} \\
		\textbf{glass-065} & 0.50  & 0.67  & 0.73  & 0.66  & 0.66  & \textbf{0.76} & \textbf{0.76} \\
		\textbf{led7digit} & 0.50  & 0.59  & 0.63  & 0.85  & 0.85  & \textbf{0.88} & \textbf{0.88} \\
		\textbf{yeast-0256} & 0.54  & 0.74  & 0.81  & 0.74  & 0.74  & \textbf{0.83} & 0.80 \\
		\textbf{yeast-0257} & 0.69  & 0.86  & 0.92  & 0.86  & 0.85  & \textbf{0.92} & 0.89 \\
		\textbf{yeast-0359} & 0.52  & 0.61  & 0.61  & 0.62  & 0.62  & 0.73  & \textbf{0.74} \\
		\textbf{Average} & 0.64  & 0.73  & 0.77  & 0.78  & 0.78  & \textbf{0.84} & 0.81 \\
		\caption{Classification recall on the 60 imbalanced datasets.}\\
	\end{longtable}%
	%


	\begin{longtable}{lrrrrrrr}\label{table3}
	
		&       &       & \multicolumn{1}{l}{\textbf{Precision}} &       &       &       &  \\
		&       & \multicolumn{1}{l}{\textbf{Knn}} &       &       & \multicolumn{1}{l}{\textbf{SVM}} &       &  \\
		& \multicolumn{1}{l}{\textbf{Orig}} & \multicolumn{1}{l}{\textbf{SMOTE}} & \multicolumn{1}{l}{\textbf{SUGAR}} & \multicolumn{1}{l}{\textbf{Orig}} & \multicolumn{1}{l}{\textbf{SMOTE }} & \multicolumn{1}{l}{\textbf{SUGAR}} & \multicolumn{1}{l}{\textbf{RusBOOST}} \\
		\textbf{ecoli-013} & 0.49  & 0.66  & 0.67  & 0.61  & \textbf{0.62} & 0.61  & 0.57 \\
		\textbf{ecoli-013} & 0.65  & 0.67  & \textbf{0.70} & 0.64  & 0.64  & 0.69  & 0.67 \\
		\textbf{ecoli1} & 0.54  & 0.55  & \textbf{0.62} & 0.54  & 0.54  & 0.59  & 0.55 \\
		\textbf{ecoli2} & 0.59  & 0.61  & \textbf{0.64} & 0.59  & 0.60  & 0.60  & 0.50 \\
		\textbf{ecoli3} & 0.75  & 0.72  & 0.74  & 0.73  & 0.74  & 0.74  & \textbf{0.78} \\
		\textbf{glass-0123} & 0.88  & 0.94  & 0.94  & 0.90  & 0.90  & \textbf{0.91} & 0.89 \\
		\textbf{glass-0162} & 0.46  & 0.50  & 0.58  & 0.65  & \textbf{0.67} & 0.61  & 0.64 \\
		\textbf{glass-0165} & 0.48  & 0.66  & 0.67  & 0.69  & 0.69  & \textbf{0.72} & \textbf{0.72} \\
		\textbf{glass0} & 0.73  & 0.73  & 0.76  & 0.79  & 0.82  & 0.77  & \textbf{0.82} \\
		\textbf{glass1} & 0.65  & 0.75  & \textbf{0.78} & 0.74  & 0.76  & \textbf{0.78} & \textbf{0.78} \\
		\textbf{glass2} & 0.59  & 0.64  & 0.64  & 0.66  & \textbf{0.67} & 0.62  & 0.62 \\
		\textbf{glass4} & 0.51  & 0.80  & 0.77  & 0.76  & 0.75  & 0.78  & \textbf{0.82} \\
		\textbf{glass5} & 0.53  & 0.68  & 0.67  & 0.69  & 0.69  & 0.69  & \textbf{0.73} \\
		\textbf{glass6} & 0.92  & 0.92  & \textbf{0.93} & 0.90  & 0.90  & 0.90  & 0.84 \\
		\textbf{haberman} & 0.56  & 0.59  & 0.61  & 0.60  & 0.57  & \textbf{0.66} & 0.58 \\
		\textbf{new-thyroid1} & 0.90  & 0.91  & 0.93  & 0.97  & \textbf{0.98} & 0.95  & 0.92 \\
		\textbf{new-thyroid2} & 0.90  & 0.95  & \textbf{0.98} & 0.96  & 0.96  & 0.95  & 0.93 \\
		\textbf{pb-134} & 0.70  & 0.89  & 0.91  & 0.89  & 0.95  & \textbf{0.92} & 0.92 \\
		\textbf{pb-0} & 0.87  & 0.87  & 0.89  & 0.84  & 0.84  & \textbf{0.89} & 0.91 \\
		\textbf{pima} & 0.63  & 0.62  & 0.65  & 0.72  & 0.71  & \textbf{0.73} & 0.70 \\
		\textbf{segment0} & 0.94  & 0.96  & 0.97  & 0.99  & \textbf{1.00} & 0.99  & 0.99 \\
		\textbf{shuttle-c24} & \textbf{0.70} & \textbf{0.70} & \textbf{0.70} & 0.69  & \textbf{0.70} & \textbf{0.70} & 0.48 \\
		\textbf{vehicle0} & 0.85  & 0.89  & 0.89  & 0.93  & \textbf{0.95} & 0.93  & 0.93 \\
		\textbf{vehicle1} & 0.56  & 0.57  & 0.61  & 0.68  & 0.69  & \textbf{0.84} & 0.67 \\
		\textbf{vehicle2} & 0.76  & 0.82  & 0.83  & 0.95  & 0.94  & \textbf{0.95} & 0.92 \\
		\textbf{vehicle3} & 0.57  & 0.56  & 0.61  & 0.68  & 0.69  & \textbf{0.84} & 0.68 \\
		\textbf{vowel0} & 0.89  & 0.93  & 0.97  & \textbf{1.00} & \textbf{1.00} & \textbf{1.00} & 0.95 \\
		\textbf{wisconsin} & 0.94  & \textbf{0.96} & 0.95  & 0.95  & \textbf{0.96} & \textbf{0.96} & 0.95 \\
		\textbf{yeast-0567} & 0.50  & 0.67  & \textbf{0.76} & 0.67  & 0.68  & 0.68  & 0.68 \\
		\textbf{yeast-1289} & 0.48  & 0.49  & \textbf{0.68} & 0.54  & 0.53  & 0.60  & 0.54 \\
		\textbf{yeast-1458} & 0.48  & 0.49  & \textbf{0.57} & 0.54  & 0.54  & 0.53  & 0.53 \\
		\textbf{yeast-17} & 0.47  & 0.59  & \textbf{0.67} & 0.62  & 0.57  & 0.62  & 0.60 \\
		\textbf{yeast-24} & 0.80  & 0.91  & \textbf{0.89} & 0.82  & 0.84  & 0.85  & 0.84 \\
		\textbf{yeast-28} & 0.68  & 0.69  & \textbf{0.72} & 0.63  & 0.56  & 0.71  & 0.61 \\
		\textbf{yeast1} & 0.63  & 0.64  & 0.67  & 0.68  & 0.68  & \textbf{0.69} & 0.70 \\
		\textbf{yeast3} & 0.82  & 0.81  & 0.82  & 0.81  & 0.82  & 0.77  & \textbf{0.85} \\
		\textbf{yeast4} & 0.64  & \textbf{0.66} & \textbf{0.66} & 0.63  & 0.62  & 0.60  & 0.60 \\
		\textbf{yeast5} & 0.79  & 0.82  & 0.83  & 0.83  & 0.82  & 0.72  & \textbf{0.86} \\
		\textbf{yeast6} & 0.66  & 0.68  & \textbf{0.72} & 0.65  & 0.64  & 0.61  & 0.60 \\
		\textbf{ecoli-0145} & 0.83  & 0.87  & 0.86  & \textbf{0.89} & 0.86  & 0.83  & 0.77 \\
		\textbf{ecoli-0-147} & 0.59  & \textbf{0.86} & \textbf{0.86} & 0.82  & \textbf{0.86} & 0.84  & 0.79 \\
		\textbf{ecoli-01475} & 0.57  & \textbf{0.90} & 0.87  & 0.86  & \textbf{0.90} & 0.84  & 0.86 \\
		\textbf{ecoli-0123} & 0.82  & 0.89  & \textbf{0.92} & 0.88  & 0.84  & 0.86  & 0.83 \\
		\textbf{ecoli-015} & 0.83  & 0.88  & \textbf{0.89} & \textbf{0.89} & 0.88  & 0.85  & 0.83 \\
		\textbf{ecoli-0234} & 0.80  & 0.88  & 0.87  & \textbf{0.89} & 0.87  & 0.87  & 0.82 \\
		\textbf{ecoli-0267} & 0.61  & 0.81  & 0.80  & \textbf{0.84} & 0.84  & 0.80  & 0.81 \\
		\textbf{ecoli-0346} & 0.87  & 0.88  & 0.88  & \textbf{0.90} & 0.88  & 0.86  & 0.81 \\
		\textbf{ecoli-0347} & 0.61  & \textbf{0.89} & 0.88  & \textbf{0.89} & \textbf{0.89} & 0.84  & 0.84 \\
		\textbf{ecoli-0345} & 0.78  & 0.83  & 0.86  & \textbf{0.87} & 0.86  & 0.86  & 0.80 \\
		\textbf{ecoli-0465} & 0.83  & 0.86  & 0.86  & \textbf{0.88} & 0.87  & \textbf{0.88} & 0.82 \\
		\textbf{ecoli-0673} & 0.63  & 0.85  & 0.84  & 0.84  & 0.84  & \textbf{0.85} & 0.82 \\
		\textbf{ecoli-0675} & 0.64  & \textbf{0.84} & 0.83  & 0.82  & 0.82  & 0.82  & 0.86 \\
		\textbf{glass-0146} & 0.46  & 0.51  & 0.60  & \textbf{0.64} & 0.67  & 0.59  & 0.61 \\
		\textbf{glass-015} & 0.45  & 0.50  & 0.55  & \textbf{0.65} & 0.67  & 0.61  & 0.62 \\
		\textbf{glass-045} & 0.48  & 0.73  & 0.79  & 0.70  & 0.70  & 0.80  & \textbf{0.82} \\
		\textbf{glass-065} & 0.46  & 0.67  & 0.74  & 0.68  & 0.68  & 0.74  & \textbf{0.76} \\
		\textbf{led7digit} & 0.46  & 0.72  & 0.77  & 0.78  & 0.78  & \textbf{0.79} & 0.76 \\
		\textbf{yeast-0256} & 0.70  & \textbf{0.88} & \textbf{0.88} & 0.71  & 0.68  & 0.75  & 0.71 \\
		\textbf{yeast-0257} & \textbf{0.95} & 0.93  & 0.93  & 0.90  & 0.88  & 0.85  & 0.87 \\
		\textbf{yeast-0359} & 0.53  & \textbf{0.71} & \textbf{0.71} & 0.59  & 0.59  & 0.64  & 0.61 \\
		\textbf{Average} & 0.67  & 0.76  & \textbf{0.78} & 0.77  & 0.77  & 0.77  & 0.75 \\
	\caption{Classification precision on the 60 imbalanced datasets.}\\
	\end{longtable}%

\end{document}